\documentclass[nohyperref]{article}

\usepackage{microtype}
\usepackage{graphicx}
\usepackage{booktabs} 
\usepackage{subcaption}
\usepackage{authblk}
\usepackage[hidelinks]{hyperref}
\usepackage{paralist, tabularx}
\usepackage[a4paper, total={6in,9in}]{geometry}
\usepackage{algorithm}
\usepackage{algpseudocode}
\usepackage{wrapfig}
\usepackage{multirow}
\usepackage[sort,compress]{natbib}

\usepackage{diagbox}
\usepackage{bbding}

\usepackage{amsmath}
\usepackage{amssymb}
\usepackage{mathtools}
\usepackage{amsthm}
\usepackage{thm-restate}
\usepackage{mathrsfs}
\usepackage[inline]{enumitem}
\usepackage[dvipsnames]{xcolor}

\usepackage{tikz}
\usepackage[capitalize,noabbrev]{cleveref}
\usepackage{amartya_ltx}
\usepackage{dp_ssl_preset}
\usepackage[textsize=tiny]{todonotes}

\newcommand{\ouralgo}{PILLAR}
\newcommand{\OURALGO}{PrIvate Learning with Low rAnk Representations}
\usepackage{array}
\newcolumntype{L}[1]{>{\raggedright\let\newline\\\arraybackslash\hspace{0pt}}m{#1}}
\newcolumntype{C}[1]{>{\centering\let\newline\\\arraybackslash\hspace{0pt}}m{#1}}
\newcolumntype{R}[1]{>{\raggedleft\let\newline\\\arraybackslash\hspace{0pt}}m{#1}}

\usepackage{bbm}
\usepackage[utf8]{inputenc}

\usepackage{nicefrac}

\title{\ouralgo: How to make semi-private learning more effective}
\newcommand\CoAuthorMark{\footnotemark[\arabic{footnote}]}
\author[1,2]{Francesco Pinto$^*$\protect\CoAuthorMark}
\author[2]{Yaxi Hu$^*$}
\author[2]{Fanny Yang}
\author[2,3]{Amartya Sanyal} %
\date{}
\affil[1]{University of Oxford, UK}
\affil[2]{ETH Z\"urich, Switzerland}
\affil[3]{Max Planck Institute for Intelligent Systems, T\"ubingen, Germany}

\begin{document}
\maketitle

\begin{abstract}
    In Semi-Supervised Semi-Private (SP) learning, the learner has
access to both  public unlabelled and private labelled data. We
propose a computationally efficient algorithm that, under mild
assumptions on the data, provably achieves significantly lower private
labelled sample complexity and can be efficiently run on real-world
datasets. For this purpose, we leverage the features extracted by
networks pre-trained on public (labelled or unlabelled) data, whose
distribution can significantly differ from the one on which SP
learning is performed. To validate its empirical effectiveness, we
propose a wide variety of experiments under tight privacy constraints
(\(\epsilon=0.1\)) and with a focus on low-data regimes. In all of
these settings, our algorithm exhibits significantly improved
performance over available baselines that use similar amounts of
public data.
    \end{abstract}

\section{Introduction}
\label{sec:intro}
In recent years, Machine Learning~(ML) models have become ubiquitous
in our daily lives. It is now common for these models to be trained on
vast amounts of sensitive private data provided by users to offer
better services tailored to their needs. However, this has given rise
to concerns regarding users' privacy, and recent
works~\citep{shokri2017membership,Ye2022EnhancedMembership,Carlini2022MIAFirstPrinciples}
have demonstrated that attackers can maliciously query ML models to
reveal private information. To address this problem, the de-facto
standard remedy is to enforce $(\epsilon,\delta)-$Differential Privacy
(DP) guarantees on the ML algorithms~\citep{dwork06dp}. However,
satisfying these guarantees often comes at the cost of the model's
utility, unless the amount of available private training data is
significantly
increased~\citep{kasiviswanathan2011whatcan,blum2005practical,beimel2013characterizing,beimel2013private,feldman2014sample}.
A way to alleviate the utility degradation is to leverage feature
extractors pre-trained on a large-scale dataset~(assumed to be public)
and whose data generating distribution can differ from the one from
which the private data is
sampled~\citep{tramer2021BetterFeatures,de2022unlocking,li2021large,kurakin2022toward}.
Training a linear classifier on top of these pre-trained features has
been shown to be among the most cost-efficient and effective
techniques~\citep{tramer2021BetterFeatures,de2022unlocking}. Gains in
utility can also be obtained if part of the private data is deemed
public: a setting known as Semi-Private (SP)
learning~\citep{alon2019limits,yu2021do,
li2022private,papernot2017semisupervised,papernot2018scalable}.

In this work, we propose a SP algorithm to efficiently learn a linear
classifier on top of features output by pre-trained neural networks.
The idea is to  leverage the public data to estimate the Principal
Components, and then to project the private dataset on the  top-$k$
Principal Components. For the task of learning linear halfspaces, this
renders the algorithm's sample complexity independent of the
dimensionality of the data, provided that the data generating
distribution satisfies a low rank separability condition, specified
in~\Cref{defn:low-dim-large-margin}. We call this class of
distributions~\emph{large margin low rank distributions}. For the
practically relevant task of image classification, we show that
pre-trained representations satisfy this condition for a wide variety
of datasets.  In line with concerns raised by the concurrent work
of~\citet{tramer2022considerations}, we demonstrate the effectiveness
of our algorithm not only on standard image classification benchmarks
used in the DP literature (i.e. CIFAR-10 and CIFAR-100) but also on a
range of datasets~(\cref{fig:alldatasets}) that we argue better
represents the actual challenges of private training. 

Indeed, in our evaluations we particularly focus on private data
distributions that deviate significantly from the pre-training ones
and on low-data regimes. We believe that testing on such relevant
benchmarks is essential to demonstrate the practical applicability of
our algorithm. Strikingly, we observe that the benefits of our
approach increase as the privacy guarantees become tighter, i.e., when
$\epsilon$ is lower. In contrast, we find that reducing the
dimensionality of the input without imposing privacy guarantees, i.e.,
when $\epsilon=\infty$, leads to a decline in the model's performance.

\begin{figure}[t]
\includegraphics[width=\linewidth]{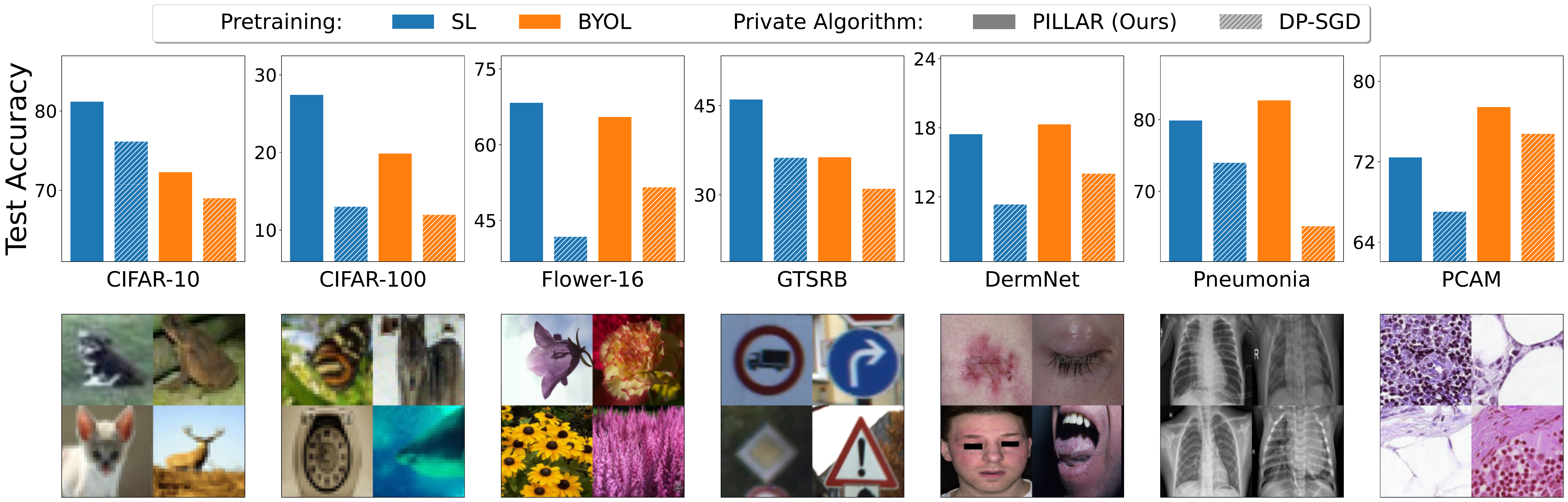}
\caption{\footnotesize We evaluate our algorithm \ouralgo~on CIFAR-10,
CIFAR-100, GTSRB, Flower-16, Dermnet, Pneumonia, and PCAM.~\ouralgo~
significantly outperforms the strongest DP baseline
DP-SGD~\citep{Abadi16dpsgd,li2022when} (see \Cref{sec:cifar-10-100}
for comparison with other baselines). Both methods use features
extracted from a ResNet-50 pre-trained on ImageNet-1K using either
Supervised Learning~(SL) or Self-Supervised
Learning~(BYOL~\citep{BYOL}) }
\label{fig:alldatasets}
\end{figure}

\noindent To summarise, our contributions are the following:
\begin{itemize}[leftmargin=1em]
    \item We propose an SP algorithm,~\ouralgo, that improves
    classification accuracy over existing baselines. Our algorithm, as
    well as baselines, use representations generated by pre-trained
    feature extractors. We demonstrate that our improvements are
    independent of the chosen pre-training strategy.
    \item  We prove that, for learning half-spaces, our algorithm
    achieves dimension-independent private labelled sample complexity
    for \emph{large margin low rank distributions}.
    \item We improve DP and SP evaluation benchmarks for image
    classification by focusing on private datasets that differ
    significantly from pre-training datasets (e.g.
    ImageNet-1K~\citep{Imagenet1K}) and with small amounts training
    data available. We enforce strict privacy constraints to better
    represent real-world challenges, and show our algorithm is
    extremely effective in these challenging settings.
    
\end{itemize}

\section{Semi-Private Learning}\label{sec:setting}
We begin by defining Differential Privacy (DP). DP ensures that the output distribution of a randomized algorithm remains stable when a single data point is modified. In this paper, a differentially private learning algorithm produces comparable distributions over classifiers when trained on neighbouring datasets. Neighbouring datasets refer to datasets that differ by a single entry. Formally,

\begin{defn}[Differential Privacy \cite{dwork06dp}]
    \label{defn:differential-privacy}
    A learning algorithm $\cA$ is $(\epsilon, \delta)$-differential private, if for any two datasets $S, S'$ differing in one entry and for all outputs $\cZ$, we have,
    \[\bP\bs{\cA(S)\in \cZ} \leq e^{\epsilon}\bP\bs{\cA(S)\in \cZ} + \delta.\] 
\end{defn} 
For $\epsilon < 1$ and $\delta = o\br{1/n}$, $(\epsilon,
\delta)$-differential privacy provides valid protection against
potential privacy attacks \citep{Carlini2022MIAFirstPrinciples}.

\paragraph{Differential Privacy and Curse of Dimensionality} Similar
to non-private learning, the most common approach to DP learning is
through Differentially Private Empirical Risk Minimization (DP-ERM),
with the most popular optimization procedure being
DP-SGD~\citep{Abadi16dpsgd} or analogous DP-variants of typical
optimization algorithms. However, unlike non-private ERM, the sample
complexity of DP-ERM suffers from a linear dependence on the
dimensionality of the
problem~\citep{chaudhuri2011differentially,bassily2014private}. Hence,
we explore slight relaxations to this definition of privacy to
alleviate this problem. We show theoretically~(\Cref{sec:theory}) and
through extensive experiments~(\Cref{sec:exp-main,sec:exp-shift}) that
this is indeed possible with some realistic assumptions on the data
and a slightly relaxed definition of privacy known as semi-private
learning that we describe below. For a discussion of broader impacts
and limitations, please refer to~\Cref{app:additional-remarks}.

\subsection{Semi-Private Learning}\label{sec:prelim} 

The concept of  semi-private learner was introduced
in~\citet{alon2019limits}. In this setting, the learning algorithm is assumed to have  access to both a private labelled and a public (labelled or unlabelled) dataset. In this work, we assume the case of only
having an \emph{unlabelled} public dataset. This specific setting has
been referred to as Semi-Supervised Semi-Private learning
in~\citet{alon2019limits}.  However, for the sake of brevity, we will
refer to it as Semi-Private learning~(SPL).

\begin{defn}[$(\alpha, \beta, \epsilon, \delta)$-semi-private learner
  on a family of distributions $\cD$]\label{defn:spl} An algorithm
  $\cA$ is said to $(\alpha, \beta, \epsilon, \delta)$-semi-privately
  learn a hypothesis class $\cH$ on a family of distributions $\cD$,
  if for any distribution $D\in \cD$, given a private labelled dataset $\Slab$
  of size $\nlab$ and a public unlabelled dataset $\Sunl$ of size $\nunl$
  sampled i.i.d. from $D$, $\cA$ is $(\epsilon ,\delta)$-DP with
  respect to $\Slab$ and outputs a hypothesis $\hat{h}$ satisfying 
  \[\bP[\bP_{(x, y)\sim D}\bs{h(x)\neq y} \leq \alpha] \geq
  1-\beta,\]where the outer probability is over the randomness of
  \(\Slab,\Sunl,\) and \(\cA\). 
  
  Further, the sample complexity $\nlab$ and $\nunl$ must be
  polynomial in $\frac{1}{\alpha},\frac{1}{\beta},$ and the size of
  the input space. In addition, $\nlab$ must also be polynomial in
  $\frac{1}{\epsilon}$ and $\frac{1}{\delta}$. The algorithm is said
  to be efficient if it also runs in time polynomial in
  \(\frac{1}{\alpha},\frac{1}{\beta},\) and the size of the input
  domain.
\end{defn}

A key distinction between our work and the previous study
by~\citet{alon2019limits} is that they examine the
distribution-independent agnostic learning setting, whereas we
investigate the distribution-specific realisable setting. On the other hand, while their algorithm is
computationally inefficient, ours can be run in time polynomial in the
relevant parameters and implemented in practice on various datasets
with state-of-the-art results. We discuss our algorithm 
in~\Cref{sec:our_algo}.

\paragraph{Relevance of Semi-Private Learning} 
In various privacy-sensitive domains such as healthcare, legal, social
security, and census data, there is often some amounts of publicly
available data in addition to the private data. For instance, the U.S.
Census Bureau office has partially released historical data before
2020 without enforcing any differential privacy
guarantees~\footnote{\url{https://www2.census.gov/library/publications/decennial/2020/census-briefs/c2020br-03.pdf}}.
It has also been observed that different data providers may have
varying levels of concerns about
privacy~\citep{PrivacyPracticesUsers}. In medical data, some patients
may consent to render some of their data public to foster research.
In other cases, data may become public due to the expiration of the
right to privacy after specific  time limits
\footnote{\url{https://www.census.gov/history/www/genealogy/decennial_census_records/the_72_year_rule_1.html}}. 

It is also very likely that this public data may be~\emph{unlabelled}
for the task at hand. For example, if data is collected to train a
model to predict a certain disease, the true diagnosis may have been
intentionally removed from the available public data to protect
sensitive information of the patients. Further, the data may had been
collected for a different purpose like a vaccine trial. Finally, the
cost of labelling may be prohibitive in some cases. Hence, when
public~(unlabelled) data is already available, we focus on harnessing
this additional data effectively, while safeguarding the privacy of
the remaining private data. We hope this can lead to the development
of highly performant algorithms which in turn can foster wider
adoption of privacy-preserving techniques.

\subsection{\ouralgo: An Efficient Semi-Private Learner}
\label{sec:our_algo}
\begin{figure}[t]
\begin{minipage}[b]{0.44\linewidth}
\begin{tikzpicture}[line width=1pt,font=\scriptsize]
    \draw[Periwinkle, very thick, pattern=north west lines, pattern color=Periwinkle!40] (-3.8,1.3) rectangle (3.6,3.4);
   \node[database,label={[align=center]below:Public\\Unlabelled\\Images},
    color=orange] (Sunl) at
    (3.,5) {} ;
    \node[database,label={[align=center]below:Private\\Labelled\\Images},database
    radius=0.5cm,database segment height=0.2cm, color=OliveGreen] (Slab) at (0.7,5) {};
    \node[database,label={[align=center]below:Public Labelled/Unlabelled},database
    radius=0.7cm,database segment height=0.3cm,text=Periwinkle,color=Periwinkle] (Spre) at (-2.,2.5) {};
    \node (feat) at (2,2.5) [draw,align=center,thick,text
    width=2cm,minimum width=3cm,minimum
    height=1cm,color=Periwinkle,text=Periwinkle,fill=Periwinkle!20] {Feature
    Extractor\\(Neural Network)};

    \node (alg) at (2,0.3) [draw,thick,minimum width=2cm,minimum
    height=1cm] {\ouralgo (\(\cA_{\epsilon,\delta}\))};
    \draw[->,color=OliveGreen, very thick] ($(Slab.south)-(-0.,0.9)$) -- ($(feat.north)+(-1.3,0)$);
    \draw[->,color=orange, very thick]  ($(Sunl.south)-(0,0.9)$) -- ($(feat.north)+(1.0,0)$);

    \draw[<->,dashed]  ($(Spre.north)$)  
    ($(Slab.west)$) arc[radius=1.5, start angle=60, end angle=200] node[below, right=-0.2cm,
    align=center, text width=2cm, pos=0.7]{Large\\Distribution Shift} ;

    \draw [<->, dashed] (3.2,5.7) arc (60:120:2)
    node[above,pos=0.5]{Small} node[below=0.2cm,pos=0.5]{Distribution
    Shift};

    \draw[->,color=OliveGreen, very thick] ($(feat.south)-(0+0.5,0)$) -- ($(alg.north)+(-0.5,0)$) node [left = -0.05cm, pos=0.85]{$\Slab$};
    \draw[->,color=Orange, very thick]  ($(feat.south)-(0-0.5,0)$) -- ($(alg.north)+(0.5,0)$)node [right = 0cm, pos=0.85]{$\Sunl$};

    \node (what) at (2,-1.2) [draw,thick,minimum width=1cm,minimum
    height=0.5cm,fill=black!20, text width=2.5cm,align=center] { \(\hw\)\\(Linear Classifier)};

    \draw[->]  ($(alg.south)$) --  ($(what.north)$);

    \draw[->,color=Periwinkle]  ($(Spre.east)-(0,0)$) --node[above]{Pre-Training} ($(feat.west)+(0,0)$);
    
\end{tikzpicture}
\end{minipage}\hfill
\begin{minipage}[b]{0.48\linewidth}
    \begin{algorithm}[H]\footnotesize
        \caption{\ouralgo: $\cA_{\epsilon, \delta}\br{k, \zeta}$}
        \label{alg:no_shift}
        \begin{algorithmic}[1]
        \State \textbf{Input:} {\color{OliveGreen}Labelled dataset $\Slab$}, {\color{Orange}Unlabelled dataset $\Sunl$}, low-dimension $k$, distributional parameter $\zeta$, high probability parameter $\beta$. 
        \State Using $\Sunl$, construct $\ecov = \sum_{x\in \Sunl}xx^{\top}/\nunl$. 
        \State Construct the
        transformation matrix $\hA$ whose $i^{\it th}$ column is the
        $i^{\it th}$ eigenvector of $\ecov$.  
        \State Project $\Slab$ with the transformation matrix
        $\hA$, 
        \[\Slab_k = \{(\hA^{\top} x, y): (x, y)\in \Slab\}.\]    
        \State Obtain $v_k = \Anoisy(\Slab_k, \ell, (\epsilon, \delta), \nicefrac{\beta}{4})$ where $\ell$ is a $\frac{1}{\zeta}$-Lipschitz loss function defined as 
        \begin{equation}
          \label{eq:loss-funtion-alg}
          \ell(w, (x, y)) = \max\left\{1 - \frac{y}{\zeta}\ba{w, x},0\right\}.
        \end{equation}
        \State \textbf{Output:} Return $\hw = \hA v_k$. 
    \end{algorithmic}
    \end{algorithm}
\end{minipage}
\caption{\textbf{Left}: Diagram describing how \ouralgo~is applied in image classification (using DP-SGD with cross-entropy loss in Line 4 of~\Cref{alg:no_shift}). \textbf{Right}:~\ouralgo~for learning linear halfspaces.}
\label{fig:diagram-our-algo}\vspace{-5pt}
\end{figure}

In this work, we propose a (semi-supervised) semi-private learning
algorithm called~\ouralgo~(\OURALGO), described
in~\Cref{fig:diagram-our-algo}. Before providing formal guarantees
in~\Cref{sec:theory}, we first describe how~\ouralgo~is applied in
practice. Our algorithm works in two stages. 

Leveraging recent practices~\citep{de2022unlocking,
tramer2021BetterFeatures} in DP training with deep neural networks, we
first use pre-trained feature extractors to transform the private
labelled and public unlabelled datasets to the representation space to
obtain the private and public representations. We use the
representations in the penultimate layer of the pre-trained neural
network for this purpose. As shown in~\Cref{fig:diagram-our-algo}, the
feature extractor is trained on large amounts of labelled or
unlabelled public data, following whatever training procedure is
deemed most suitable.  For this paper, we pre-train a ResNet-50 using
supervised training~(SL), self-supervised training~(BYOL~\citep{BYOL}
and MocoV2+~\citep{MocoV2+}), and semi-supervised training (SemiSL and
Semi-WeakSL~\citep{SemiSL}) on ImageNet. In the main body, we only
focus on SL and BYOL pre-training.  As we discuss extensively
in~\Cref{sec:add_exp}, our algorithm is effective independent of the
choice of the pre-training algorithm.  In addition, while the private
and public datasets are required to be from the same~(or similar)
distribution, we show that the pre-training dataset can come from a
significantly different distribution. In fact, we use ImageNet as the
pre-training dataset for all our experiments even when the
distributions of the public and private datasets range from
CIFAR-10/100 to histological and x-ray images as shown
in~\Cref{fig:alldatasets}. Recently,~\citet{gu2023choosing} have
explored the complementary question of how to choose the right
pre-training dataset.

In the second stage, \ouralgo~takes as input the feature
representations of the private labelled and public unlabelled
datasets, and feeds them to~\Cref{alg:no_shift}. We denote these
datasets of representations as \(\Slab\) and \(\Sunl\) respectively.
Briefly, ~\Cref{alg:no_shift} projects the private dataset \(\Slab\)
onto a low-dimensional space spanned by the top principal components
estimated with \(\Sunl\), and then applies gradient-based private
algorithms (e.g.  Noisy-SGD~\citep{bassily2014private}) to learn a
linear classifier on top of the projected
features.~\Cref{alg:no_shift} provides an implementation of
\ouralgo~with Noisy-SGD, whereas in our experiments we show that
commonly used DP-SGD~\citep{Abadi16dpsgd} is also effective.

\section{Theoretical Results}
\label{sec:theory}
In this section, we first describe the assumptions under which we
provide our theoretical results and show they can be motivated both
empirically and theoretically. Then, we show a dimension-independent
sample complexity bound for \ouralgo~under the mentioned assumptions.

\subsection{Problem setting}
\label{sec:dist_assump}

Our theoretical analysis focuses on learning linear halfspaces
$\linear{d}$ in \(d\) dimensions. Consider the instance space
\(\instspace_d = B_2^{d} = \bc{x\in \bR^{d}: \norm{x}_2 =1}\) as the
$d$-dimensional unit sphere and the binary label space \(\cY=\bc{-1,1
}\). In practice, the instance space is the~(normalized)
representation space obtained from the pre-trained network. The
hypothesis class of linear halfspaces is \[\linear{d} = \bc{f_w(x) =
\text{sign}\br{\ba{w, x}}\vert w \in B_2^d }.\]

We consider the setting of distribution-specific learning, where our
family of distributions admits a large margin linear classifier that
contains a significant projection on the top principal components of
the population covariance matrix. We formalise this as
\(\br{\gamma,\xi_k}\)-Large margin low rank distributions.
\begin{defn}[\(\br{\gamma,\xi_k}\)-Large margin low rank distribution]
    \label{defn:low-dim-large-margin}
    A distribution $D$ over $\cX_d\times \cY$ is a
    \(\br{\gamma, \xi_k}\)-Large margin low rank distribution if there exists
    $\tw\in B_2^d$ such that
    \begin{itemize}
      \item \(\bP_{(x, y)\sim D}\bs{\frac{y\ba{\tw,
      x}}{\norm{\tw}_2\norm{x}_2}\geq \gamma}=1\) ~\enskip (Large-margin),
      \item $\norm{\A\A^{\top}\tw}_2\geq 1-\xi_k$~\enskip (Low-rank separability). 
    \end{itemize}
    where $\A$ is a $d\times k$ matrix whose columns are the top $k$
    eigenvectors of \(\bE_{X\sim D_X}\bs{X^\top X}\).
  \end{defn}

  \begin{wrapfigure}{r}{0.3\linewidth}
    \vspace{-10pt}
    \includegraphics[width=\linewidth]{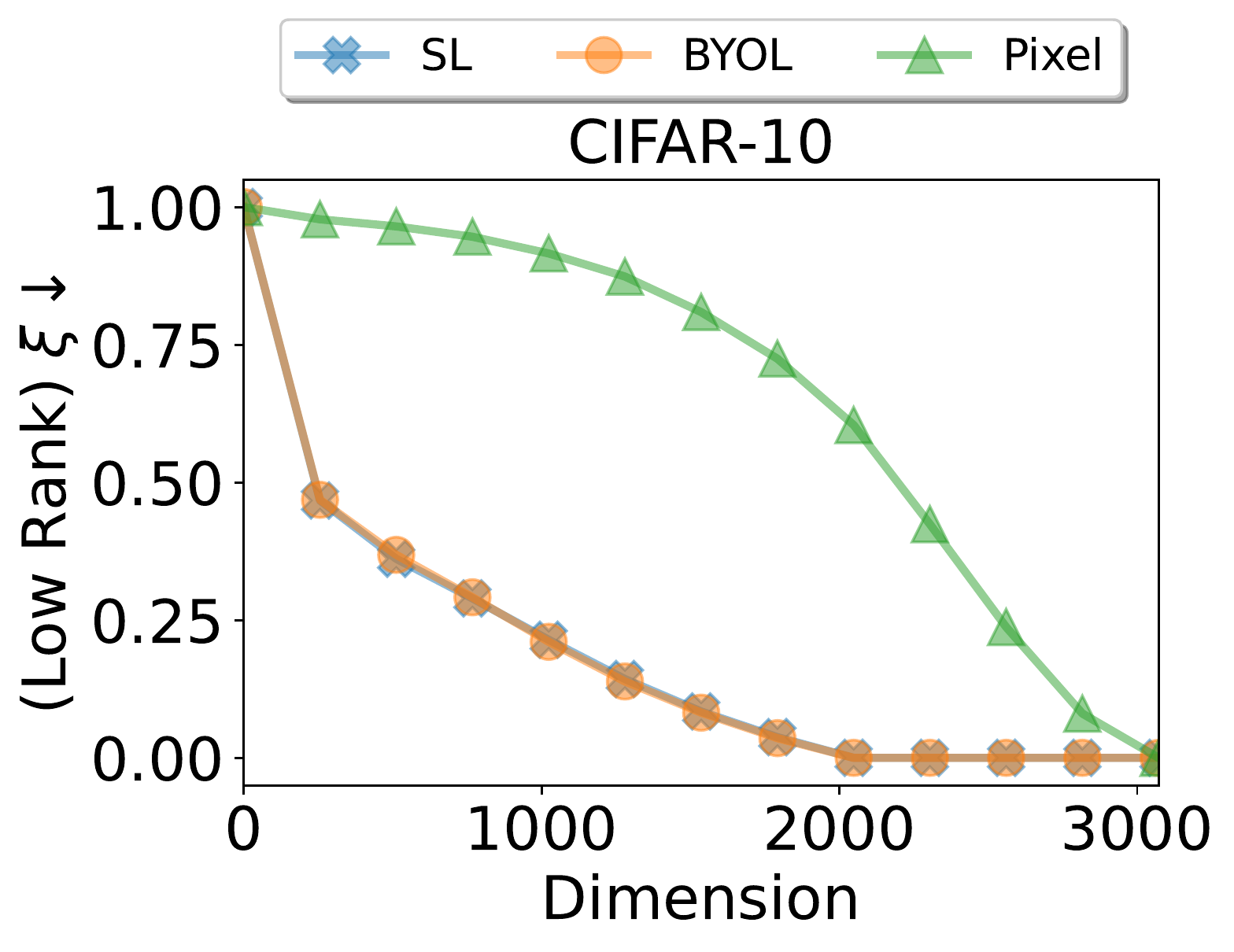}
        \includegraphics[width=\linewidth]{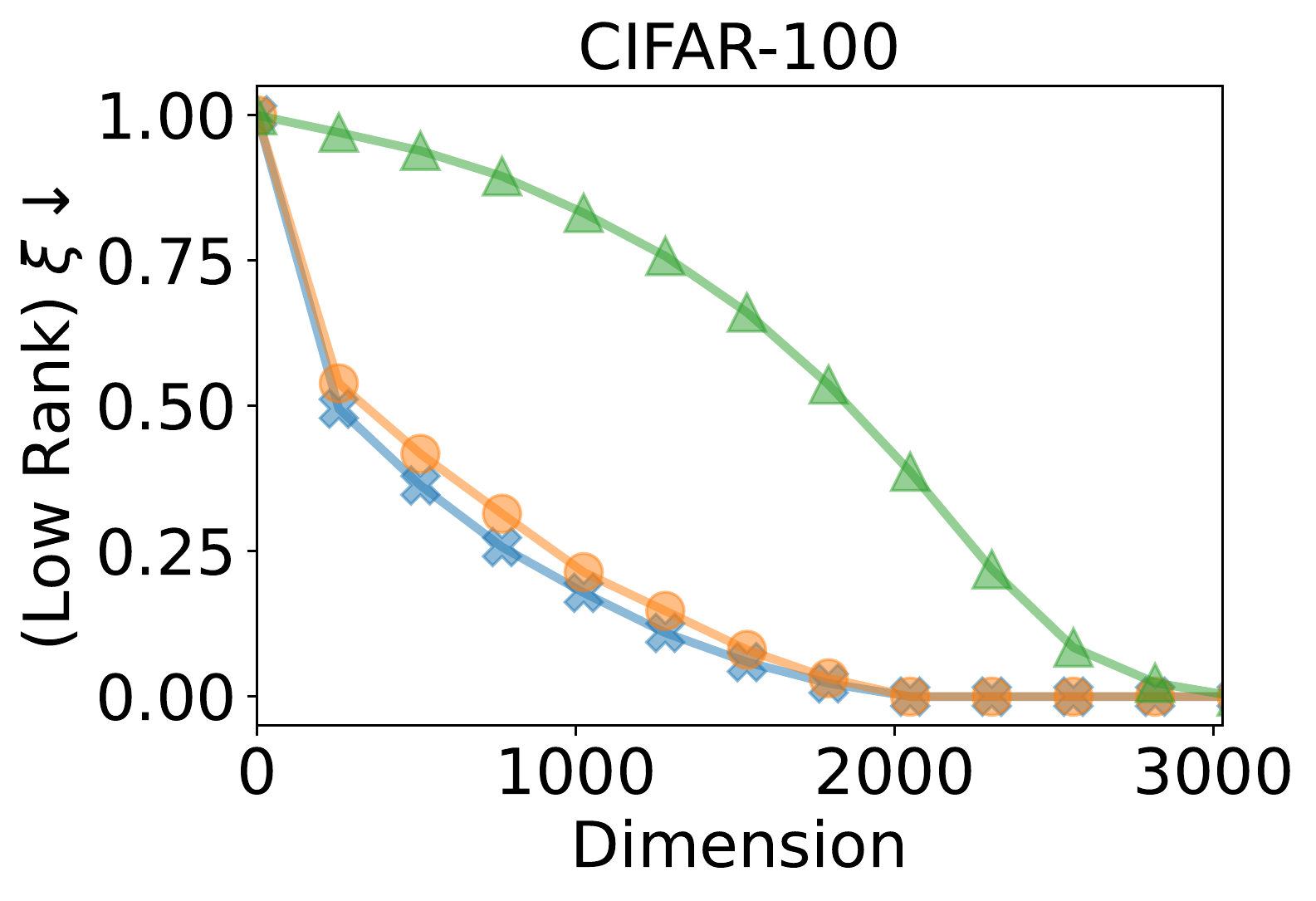}
        \caption{\small Estimate of $\xi$ for linear classifiers trained on embeddings of two CIFAR-10 and CIFAR-100 classes, extracted from pre-trained ResNet50s, as well as the raw images~(Pixel).}
        \label{fig:xi}
        \vspace{-40pt}
    \end{wrapfigure}

It is worth noting that for every distribution that admits a positive
margin~\(\gamma\), the low-rank separability condition is
automatically satisfied for all \(k\leq d\) with some \(\xi_k\geq 0\).
However, the low rank separability is helpful for learning, only if it
holds for a small \(k\) and small \(\xi_k\) simultaneously. These
assumptions are both theoretically and empirically realisable.
Theoretically, we show in~\Cref{defn:large-margin-gmm} that a class of
commonly studied Gaussian mixture distributions satisfies these
properties. Empirically, we show in~\Cref{fig:xi} that pre-trained
features satisfy these properties with small $\xi$ and $k$. 

\begin{example}\label{defn:large-margin-gmm} A distribution $D$ over
$\cX\times \cY$ is a $(\theta, \sigma^2)$-Large margin Gaussian
mixture distribution if there exists $\tw, \mu\in B_2^d$, such that
$\ba{\mu, \tw} = 0$, the conditional random variable  \(X\vert y\) is
distributed according to a normal distribution with mean \(\mu y\) and
covariance matrix \(\theta \tw \br{\tw}^{\top} + \sigma^2I_d\) and
\(y\in\{-1,1\}\) is distributed uniformly.
\end{example}

For any $\theta, \sigma^2 = O(\nicefrac{1}{\sqrt{d}})$, it is easy to
 see that this family of distributions satisfies the large margin low
 rank properties in~\Cref{defn:low-dim-large-margin} for $k = 2$ and
 $\xi=0$.  
Next, we show empirically that these assumptions approximately hold on
the features of pre-trained feature extractors  we use in this work.

\label{sec:bridge-th-exp}

\paragraph{Pre-trained features are almost Large-Margin and  Low-Rank}

~\Cref{fig:xi} shows that feature representations of CIFAR-10 and
CIFAR-100 obtained by various pre-training strategies approximately
satisfy the conditions of~\Cref{defn:low-dim-large-margin}. To verify
the low-rank separability assumption, we first train a binary linear
SVM \(\tw\) for a pair of classes on the representation space and
estimate \(\xi_k=1-\norm{A_kA_k^{\top}\tw}_2\)~as defined
in~\Cref{defn:low-dim-large-margin}. We also compute $\xi_k$ when
$\tw$ is trained on the pixel space~\footnote{The estimate of $\xi_k$
on pixel space should be taken with caution since classes are not
linearly separable in the pixel space thereby only approximately
satisfying the Large Margin assumption.}. As shown in \Cref{fig:xi},
images in the representation space are better at satisfying the
low-rank separability assumption compared to images in the pixel
space.

\subsection{Private labelled sample complexity analysis} 

In this section, we bound the sample complexity of \ouralgo~for
learning linear halfspaces. Our theoretical analysis relies on a
scaled hinge loss function that depends on the privacy parameters
$\epsilon, \delta$, as well as an additional parameter $\zeta$,
characterized by $\gamma$ and $\xi_k$.  We also denote the gap between
the \(k^\mathrm{th}\) and the \({k+1}^{\mathrm {th}}\) eigenvalue of
the population covariance matrix as \(\deltak\). 

\Cref{thm:no-shift} shows that if the private and public datasets come
from the same large-margin low rank distribution, then
\ouralgo~defined in~\Cref{alg:no_shift} is both
\(\br{\epsilon,\delta}\)-DP with respect to the private dataset as
well as accurate, with relatively small number of private labelled
data. As motivated in~\cref{sec:dist_assump}, the feature
representations of images are usually from large-margin low rank
distributions. Thus, in practical implementation, the private and
public datasets refer to private and public representations, as shown
in~\Cref{fig:diagram-our-algo}.

\begin{restatable}{thm}{noshift}
  \label{thm:no-shift}
   Let \(k\leq d\in\bN, \gamma_0\in (0, 1),\) and  \(\xi_0 \in [0,
   1)\) satisfy $\xi_0<1 - \nicefrac{\gamma_0}{10}$. Consider the
   family of distributions $\cD_{\gamma_0, \xi_0}$ which consists of
   all \(\br{\gamma,\xi_k}\)-large margin low rank distributions over
   \(\instspace_d\times\cY\), where $\gamma \geq \gamma_0$ and $\xi_k
   \leq \xi_0$.  For any $\alpha \in \br{0, 1}, \beta \in \br{0,
   \nicefrac{1}{4}}$, $\epsilon\in \br{0, \nicefrac{1}{\sqrt{k}}}$,
   and $\delta\in (0, 1)$, $\cA_{\epsilon, \delta}(k, \zeta)$,
   described by~\Cref{alg:no_shift} is an $(\alpha, \beta, \epsilon,
   \delta)$-semi-private learner for linear halfspaces $\linear{d}$ on
   $\cD_{\gamma_0, \xi_0}$ with sample complexity 
   \begin{equation*}\label{eq:lin-priv-sample-complexity}
    \nunl = \bigO{\frac{\log \nicefrac{2}{\beta}}{\gamma_0^2\deltak^2}}, \nlab = \widetilde{O}\br{\frac{\sqrt{k}}{\alpha\epsilon\zeta}}\end{equation*}
    where
    $\zeta = \gamma_0\br{1-\xi_0-0.1\gamma_0}$. 
\end{restatable}
Note that while~\Cref{thm:no-shift} only guarantees $(\epsilon,
\delta)$-DP on the set of private representations
in~\Cref{fig:diagram-our-algo}, this guarantee can also extend to
$(\epsilon, \delta)$-DP on the private labelled image dataset.
See~\Cref{sec:appendix-no-shift} for more details. For large margin
Gaussian mixture distributions, \Cref{thm:no-shift} implies that
\ouralgo~leads to a drop in the private sample complexity from
$O(\sqrt{d})$ to $O(1)$. The formal statement with the proof is
provided in~\Cref{sec:appendix-gmm}.
\paragraph{Tolerance to distribution shift} In real-world scenarios,
there is often a distribution shift between the labeled and unlabeled
data. For instance, when parts of a dataset become public due to
voluntary sharing by the user, there is a possibility of a
distribution shift between the public and private datasets. To handle
this situation, we use the notion of Total Variation (TV) distance.
Let \(D^L\) be the labelled data distribution, \(D^L_X\) be its
marginal distribution on \(\cX\), and \(D^U\) be the unlabelled data
distribution.  We say $D^L_X$ and $D^U$ have \(\eta\)-bounded TV
distance if 
\[TV(D^U, D^L_X) = \sup_{A\subset \cX} \vert D^U(A) - D^L_X(A)\vert
\leq \eta.\]

We extend the definition of semi-private learner to this setting. An
algorithm $\cA$ is an \textit{$\eta$-TV tolerant $(\alpha, \beta,
\epsilon, \delta)$-semi-private learner} if it is a $(\alpha, \beta,
\epsilon, \delta)$-semi-private learner when the labelled and
unlabelled distributions have $\eta$-bounded TV.

\begin{restatable}{thm}{withshift}\label{thm:with-shift} Let \(k\leq
  d\in \bN,\gamma_0\in (0, 1),\) and \(\xi_0\in [0, 1)\) satisfy
  $\xi_0 < \nicefrac{1}{2} - \nicefrac{\gamma}{10}$. Consider the
  family of distributions $\cD_{\gamma_0, \xi_0}$ consisting of all
  \(\br{\gamma,\xi_k}\)-large margin low rank distributions over
  \(\instspace_d\times\cY\) with $\gamma \geq \gamma_0$, $\xi_k\leq
  \xi_0$, and small third moment. For any $\alpha \in \br{0, 1}, \beta
  \in \br{0, \nicefrac{1}{4}}$, $\epsilon\in \br{0,
  \nicefrac{1}{\sqrt{k}}}, \delta\in (0, 1)$ and $\eta \in [0,
  \nicefrac{\deltak}{14})$, $\cA_{\epsilon, \delta}(k, \zeta)$,
  described by ~\Cref{alg:no_shift}, is an $\eta$-TV tolerant
  $(\alpha, \beta, \epsilon, \delta)$-semi-private learner of the
  linear halfspace $\linear{d}$ on $\cD_{\gamma_0, \xi_0}$ with sample
  complexity  
  \begin{equation*}
    \nunl = \bigO{\frac{\log\frac{2}{\beta}}{\br{\gamma_0\deltak}^2}},
    \nlab = \tilde{O}\br{\frac{\sqrt{k}}{\alpha\epsilon\zeta}}
    \end{equation*}
  where $\zeta = \gamma_0(1-\xi_0 - 0.1\gamma_0 -\nicefrac{7\eta}{\deltak})$. 
\end{restatable}
\Cref{thm:with-shift} provides a stronger version of
\Cref{thm:no-shift}, which can tolerate a distribution shift between
labeled and unlabeled data.  Refer to~\Cref{sec:appendix-with-shift}
for a formal definition of $\eta$-TV tolerant $(\alpha, \beta,
\epsilon,\delta)$-learner and a detailed version
of~\Cref{thm:with-shift} with its proof.

\subsection{Comparison with existing theoretical results and discussion}
\label{sec:theory-comp}
Existing works have offered a variety of techniques for achieving
dimension-independent sample complexity. In this section, we review
these works and compare them with our approach.  

\paragraph{Generic private algorithms}

~\citet{BassilyST14} proposed the Noisy SGD algorithm \(\Anoisy\) to
privately learn linear halfspaces with margin $\gamma$ on a private
labelled dataset of size
$O(\nicefrac{\sqrt{d}}{\alpha\epsilon\gamma})$. Recently,
~\citet{li2022when} showed that DP-SGD, a slightly adapted version of
\(\Anoisy\), can achieve a dimension independent error bound under a
low-dimensionality assumption termed as Restricted Lipschitz
Continuity (RLC). However, these methods cannot utilise public
unlabelled data. Moreover, the low dimensional assumption with RLC is
more stringent compared with our low-rank separability assumption.  In
contrast, generic semi-private learner in~\citet{alon2019limits}
leverages unlabelled data to reduce the infinite hypothesis class to a
finite $\alpha$-net and applies exponential
mechanism~\citep{mcsherry2007mechanism} to achieve $(\epsilon, 0)$-DP.
Nonetheless, it is not computationally efficient and still requires a
dimension-dependent labelled sample complexity
$\bigO{\nicefrac{d}{\alpha\epsilon}}$.
 
\paragraph{Dimension reduction based private algorithms} Perhaps, most
relevant to our work,~\citep{nguyen19jl} applies Johnson-Lindenstrauss
(JL) transformation to reduce the dimension of a linear halfspace with
margin $\gamma$ from $d$ to $\bigO{\nicefrac{1}{\gamma}}$ while
preserving the margin in the lower-dimensional space. Private learning
in the transformed low-dimensional space requires
$\bigO{\nicefrac{1}{\alpha\epsilon\gamma^2}}$ labelled samples. Our
algorithm removes the quadratic dependence on the inverse of the
margin but pays the price of requiring the linear separator to align
with the top few principal components of the data. For example, a
Gaussian mixture distribution (\Cref{defn:large-margin-gmm}) satisfies
large margin low rank property with parameters $\xi_k = 0$ and $k =
1$. \Cref{corollary:gmm-full} shows that our algorithm requires
labelled sample complexity \(O(\nicefrac{1}{\alpha\epsilon\gamma})\)
instead of \(O(\nicefrac{1}{\alpha\epsilon\gamma^2})\) required
by~\citet{nguyen19jl}. Low dimensionality assumptions have also
successfully been exploited in differentially private data
release~\citep{blum2013fast,donhauser2023sample}, however their work
is unrelated to ours and we do not compare with this line of work.

Another approach to circumvent the dependency on the dimension is to
apply dimension reduction techniques directly to the gradients. For
smooth loss functions with $\rho$-Lipschitz and $G$-bounded
gradients,~\citet{zhou21pcasgd} showed that applying PCA in the
gradient space of DP-SGD~\citep{Abadi16dpsgd} achieves
dimension-independent labelled sample complexity $\bigO{\frac{k\rho
G^2}{\alpha\epsilon} + \frac{\rho^2 G^4\log d}{\alpha}}$. However,
this algorithm is computationally costly as it applies PCA in every
gradient-descent step to a matrix whose size scales with the number of
parameters.~\citet{kasiviswanathan21a} proposed a computationally
efficient method by applying JL transformation in the gradient space.
While their method can eliminate the linear dependence of DP-SGD on
dimension when the parameter space is the $\ell_1$-ball, it leads to
no improvement for parameter space being the $\ell_2$-ball as in our
setting. Gradient Embedding Perturbation (GEP) by~\citet{yu2021do} is
also computationally efficient. However, their analysis yields
dimension independent guarantees only when a strict low-rank
assumption of the gradient space is satisfied. We discuss these works
and compare the assumptions in more detail
in~\Cref{sec:appendix-comparison}.

\paragraph{Other private algorithms} Boosting is another method for
improving the utility-privacy tradeoff.~\citet{bun20boosting} analyzed
private boosting for learning linear halfspaces with margin $\gamma$.
They proposed a weak learner that is insensitive to the label noise
and can achieve dimension independent sample complexity
$O(\nicefrac{1}{\alpha\epsilon\gamma^2})$, matching the sample
complexity achieved by applying JL transformation~\citep{nguyen19jl}.

\paragraph{Non-private learning and dimensionality reduction} It is
interesting to note that our algorithm may not lead to a similar
improvement in the non-private case. We show a dimension-independent
Rademacher-based labelled sample complexity bound for non-private
learning of linear halfspaces. We use a non-private version of
\cref{alg:no_shift} by replacing Noisy-SGD with Gradient Descent using
the same loss function. As before, for any $\gamma_0\in (0, 1),
\xi_0\in (0, 1)$, let $\cD_{\gamma_0, \xi_0}$ be the family of
distributions consisting of all $(\gamma, \xi_k)$-large margin low
rank distributions with $\gamma \geq \gamma_0$ and $\xi_k\leq \xi_0$.

\begin{proposition}[Non-DP learning]
    \label{thm:nondp}
        For any $\alpha, \beta\in (0, \nicefrac{1}{4})$, and
    distribution \(D\in\cD_{\gamma_0, \xi_0}\), given a labelled
    dataset of size \(\tilde{O}\br{\nicefrac{1}{\zeta\alpha^2}}\) and
    unlabelled dataset of size
    \(\bigO{\nicefrac{\log\frac{2}{\beta}}{\br{\gamma_0\deltak}^2}}\),
    the non-private version of $\cA(k, \zeta)$ produces a linear
    classifier \(\hat{w}\) such that with probability \(1-\beta\)
    \[\bP_{D}\bs{y\ip{\hat{w}}{x}<0}<\alpha,\] where $\zeta =
    \gamma_0(1-\xi_0 - 0.1\gamma_0)$.  
\end{proposition} 
    
    The result follows directly from the uniform convergence of linear
    halfspaces with Rademacher complexity. For example, refer to
    Theorem 1 in~\citet{awasthi2020rademacher}. The labelled sample
    complexity in the above result shows that non-private algorithms
    do not significantly benefit from decreasing
    dimensionality\footnote{However, this bound uses a standard
    Rademacher complexity result and may be lose. A tighter complexity
    bound may yield some dependence on the projected dimension.}. We
    find this trend to be true in all our experiments
    in~\Cref{fig:neardist,fig:farshift}. In summary, this section has
    showed that our computationally efficient algorithm, under certain
    assumptions on the data, can yield dimension independent private
    sample complexity. We also show through a wide variety of
    experiments that the results transfer to practice in both common
    benchmarks as well as many newly designed challenging settings.

\section{Results on Standard Image Classification Benchmarks}
\label{sec:exp-main}
 In this section, we report performance of~\ouralgo~on two standard
 benchmarks~(CIFAR-10 and CIFAR-100~\citep{cifar10}) for private image
 classification.  In particular, we show how its performance changes
 with varying dimensions of projection \(k\) for varying privacy
 parameter \(\epsilon\).

\subsection{Evaluation on CIFAR-10 and CIFAR-100}
\label{sec:cifar-10-100}
\paragraph{Experimental setting}~The resolution difference between
ImageNet-1K and CIFAR images can negatively impact the performance of
training a linear classifier on pre-trained features. To mitigate this
issue, we pre-process the CIFAR images using the ImageNet-1K
transformation pipeline, which increases their resolution and leads to
significantly improved performance. This technique is consistently
applied throughout the paper whenever there is a notable resolution
disparity between the pre-training and private datasets. For further
details and discussions on pre-training at different resolutions,
please refer to~\Cref{app:details-cifar10-100}.

We diverge from previous studies in the literature, such as those
conducted
by~\citet{de2022unlocking,tramer2021BetterFeatures,kurakin2022toward},
by not exclusively focusing on values of $\epsilon > 1$. While a
moderately large $\epsilon$ can be insightful for assessing the
effectiveness of privately training deep neural networks with
acceptable levels of accuracy, it is important to acknowledge that a
large value of $\epsilon$ can result in loose privacy
guarantees.~\citet{dwork2011firm} emphasizes that reasonable values of
$\epsilon$ are expected to be less than 1.
Moreover,~\citet{yeom2018privacy} and~\citet{nasr2021adversary} have
already highlighted that $\epsilon > 1$ leads to loose upper bounds on
the success probability of membership inference attacks. Consequently,
we focus on $\epsilon \in {0.1, 0.7, \infty}$, where $\epsilon=\infty$
corresponds to the public training of the linear classifier. However,
we also present results for $\epsilon=1,2$
in~\Cref{app:details-cifar10-100-large-eps}.

\paragraph{Reducing dimension of projection~\(k\) helps private
 learning} In~\Cref{fig:neardist}, we present the test accuracy of
 private and non-private trianing on CIFAR-10 and CIFAR-100 as the
 dimensionality of projection~(PCA dimension) varies, with an initial
 embedding dimension of $k=2048$. The principal components are
 computed on a public, unlabelled dataset that constitutes 10\% of the
 full dataset, as allowed by Semi-Private Learning
 in~\Cref{defn:private-learning}. Our results demonstrate that private
 training benefits from decreasing dimensionality, while non-private
 training either performs poorly or remains stagnant. For example,
 using the SL feature extractor at $\epsilon=0.1$ on CIFAR-10, the
 test accuracy of private training reaches $81.21\%$ when $k=40$,
 compared to $76.9\%$ without dimensionality reduction.  Similarly,
 for CIFAR-100 with the SL feature extractor at $\epsilon=0.7$, the
 accuracy drops from $55.98\%$ at $k=200$ to $50.83\%$ for the full
 dimension.
 \begin{wrapfigure}{r}{0.5\linewidth}\centering
  \center 
  \includegraphics[width=0.99\linewidth]{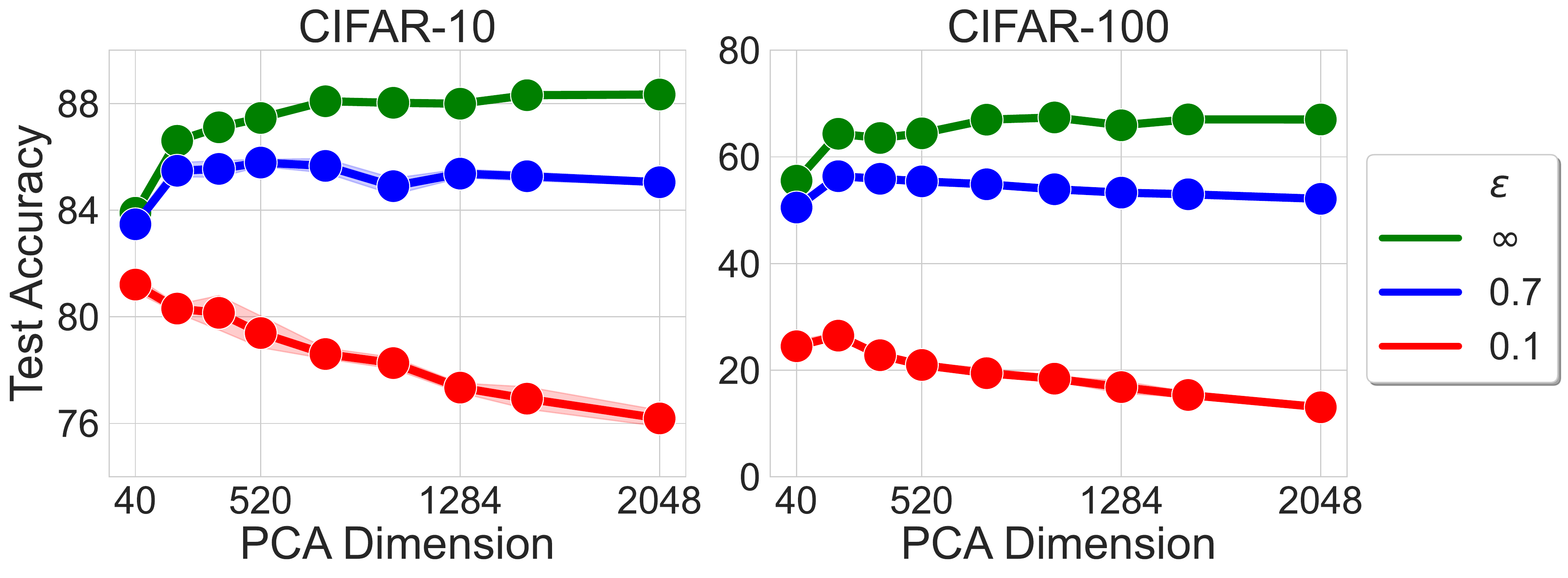}
  \caption{\small DP training of linear classifier
on SL pre-trained feature using the PRV accountant. For non-DP
training ($\epsilon=\infty$), accuracy increases as dimension
increases; opposite occurs for DP training. For results on additional
feature-extractors see~\Cref{sec:add_exp}.}
\label{fig:neardist}
\end{wrapfigure}
This observed dichotomy between private and non-private learning in
terms of test accuracy and projection dimension aligns
with~\Cref{thm:no-shift} and~\Cref{thm:nondp}
in~\Cref{sec:theory-comp}.~\Cref{thm:no-shift} indicates that that the
private test accuracy improves as the projection dimension decreases,
as depicted in~\Cref{fig:neardist}. For non-private training with
moderately large dimension,~(\(k\geq 520\)), the test accuracy remains
largely constant. We discuss this theoretically in~\Cref{thm:nondp}
~\Cref{sec:theory-comp}. The decrease in non-private accuracy for very
small values of \(k\) is attributed to the increasing approximation
error~(i.e.  how well can the best classifier in $k$ dimensions
represent the ground truth). This difference in behaviour between
private and non-private learning for decreasing \(k\) values is
observed consistently in all our experiments and is one of the main
contributions of this paper. While we have demonstrated the
effectiveness of our algorithm on the CIFAR-10 and CIFAR-100
benchmarks, as discussed in~\Cref{sec:exp-shift}, we acknowledge that
this evaluation setting may not fully reflect the actual objectives of
private learning.
\renewcommand{\arraystretch}{1.2}
\begin{table}[t]\small
  \begin{tabular}{l|c|cc|cc} 
    \toprule
    Pre-training& &\multicolumn{2}{c}{SL} & \multicolumn{2}{c}{BYOL} \\[0.3em]
    Private Algorithm &Additional Public Data& CIFAR10 & CIFAR100 & CIFAR10 & CIFAR100 \\
    \midrule
    DP-SGD~\citep{Abadi16dpsgd}& None& 84.89 & 50.65 & 79.85 & 46.12\\
    JL~\citep{nguyen19jl}& None& 84.40 & 50.56 & 78.92& 43.13\\
    AdaDPS~\citep{li2022private}& Labelled & 83.22 &  39.38 & 77.76  & 33.92\\
    GEP~\citep{yu2021do}&Unlabelled & 84.54 & 45.22 & 78.56 & 35.47 \\
    \textbf{OURS}& Unlabelled & \textbf{85.89} &  \textbf{55.86}& \textbf{80.89} & \textbf{46.94}\\
    \bottomrule
    \end{tabular}
    \caption{Comparison against baselines using publicly available
    code with comparable training settings. Here, using the RDP
    accountant for $\epsilon=0.7$. See~\Cref{app:details-comp-exp} for
    a discussion.}
    \label{tab:baselines}
\end{table}

\renewcommand{\arraystretch}{1.}
\subsection{Comparison with Existing Methods}
\label{sec:comp-exp}

    To ensure a fair comparison, we focus on the setting where
$\epsilon=0.7$ and use the RDP accountant~\citep{mironov2017renyi}
instead of the PRV accountant~\citep{PRVAccountant} employed
throughout the rest of the paper. For a comprehensive discussion on
the rationale behind this choice, implementation details, and the
cross-validation ranges for hyper-parameters across all methods, refer
to~\Cref{app:details-comp-exp}

\paragraph{Baselines} We consider the following
baselines:~\begin{enumerate*}[label=\roman*)]\item\emph{DP-SGD}
\citep{Abadi16dpsgd,li2022when}: Trains a linear classifier privately
using DP-SGD on the pre-trained features.
\item~\emph{JL} \citep{nguyen19jl}: Applies a Johnson-Lindenstrauss
(JL) transformation (without utilizing public data) to reduce the
dimensionality of the features. We cross-validate various target
dimensionalities and report the results for the most accurate one.
\item~\emph{AdaDPS} \citep{li2022private}: Utilizes the public
\emph{labeled} data to compute the pre-conditioning matrix for
adaptive optimization algorithms. Since our algorithm does not require
access to labels for the public data, this comparison is deemed
unfair.\item \emph{GEP} \citep{yu2021do}: Employs the public
\emph{unlabeled} data to decompose the private gradients into a
low-dimensional embedding and a residual component, subsequently
perturbing them with noise with different  variance.\end{enumerate*}

Whenever public data is utilized, we employ \(10\%\) of the training
data as public and remaining data as private.  The official
implementations of AdaDPS and GEP are used for our comparisons.
In~\Cref{sec:pate}, we discuss
PATE~\citep{papernot2017semisupervised,papernot2018scalable} and the
reasons for not including it in our comparisons. For a detailed
comparison with the work of~\citet{de2022unlocking}, including the use
of a different feature extractor to ensure a fair evaluation, we refer
to~\Cref{app:details-cifar10-100}, where we demonstrate that our
method is competitive, if not superior, while being significantly
computationally more efficient.

\paragraph{Results} In Table~\ref{tab:baselines}, we compare our
approach with other methods in the literature. Our results suggest
that reducing dimensionality by using the JL transformation is
insufficient to even outperform~DP-SGD. This may be attributed
to the higher sample size required for their bounds to provide
meaningful guarantees.  Similarly, employing public data to
pre-condition an adaptive optimizer does not result in improved
performance for AdaDPS.  Additionally, GEP also does not
surpass DP-SGD, potentially due to its advantages being more
pronounced in scenarios with much higher model dimensionality.
Considering their lower performance on this simpler benchmark compared
to~DP-SGD and the extensive hyperparameter space that needs to
be explored for most baselines, we focus on the strongest baseline
(DP-SGD) for more challenging and realistic evaluation settings
in~\Cref{sec:exp-shift}.

\section{Experimental Results Beyond Standard Benchmarks}
\label{sec:exp-shift}
In line with concurrent work~\citep{tramer2022considerations}, we
raise concerns regarding the current trend of utilizing pre-trained
feature extractors~\citep{de2022unlocking,tramer2021BetterFeatures}
for differentially private training. It is common practice to evaluate
differentially private algorithms for image classification by
pre-training on ImageNet-1K and performing private fine-tuning on
CIFAR datasets~\citep{de2022unlocking,tramer2021BetterFeatures}.
However, we argue that this approach may not yield generalisable
insights for privacy-sensitive scenarios. Both ImageNet and CIFAR
datasets primarily consist of everyday objects, and the label sets of
ImageNet are partially included within CIFAR. Such a scenario is
unrealistic for many privacy-sensitive applications, such as medical,
finance, and satellite data, where a large publicly available
pre-training dataset with similar characteristics to the private data
may not be accessible.

\begin{figure}[t]
    \centering
    \includegraphics[width=0.9\linewidth]{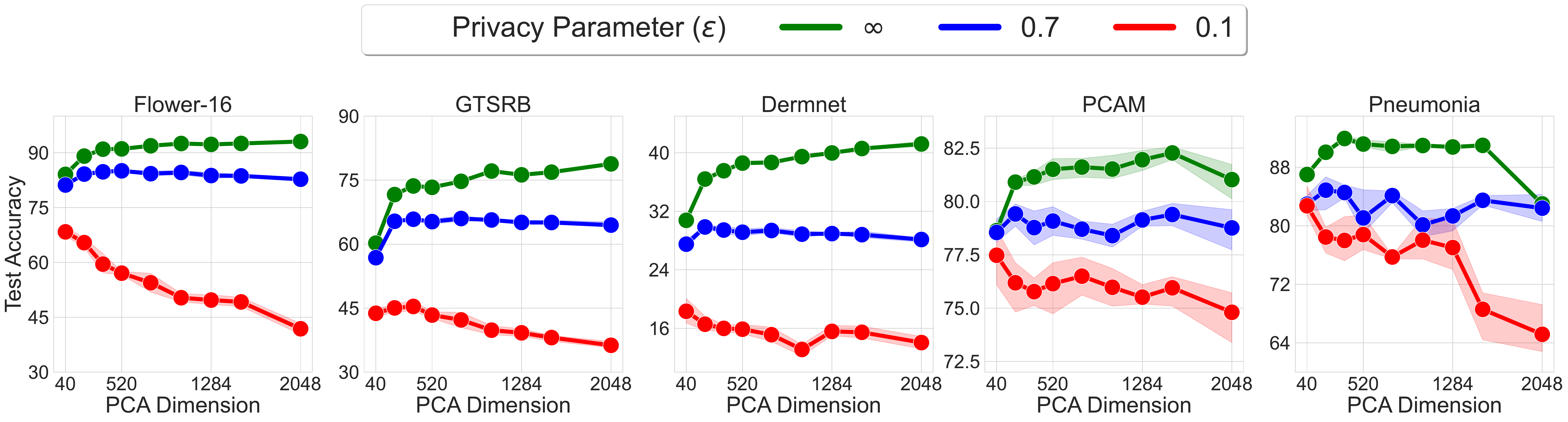}
    \caption{\small Test Accuracy of DP classification on Flower-16,
    GTSRB, Dermnet, PCAM, and Pneumonia for best pre-training
    algorithm~(SL pre-training for Flower-16 and GTSRB and BYOL for
    the remaining.). For results on additional feature-extractors refer to
    \Cref{sec:add_exp}.}
    \label{fig:farshift}
    \end{figure}

Moreover, public datasets are typically large-scale and easily scraped
from the web, whereas private data is often collected on a smaller
scale and subject to legal and competitive constraints, making it
difficult to combine with other private datasets. Additionally,
labeling private data, particularly in domains such as medical or
biochemical datasets, can be costly. Therefore, evaluating the
performance of privacy-preserving algorithms requires examining their
robustness with respect to small dataset sizes. In order to address
these considerations, we assess the performance of our algorithm on
five additional datasets that exhibit varying degrees of distribution
shift compared to the pre-training set, as described
in~\Cref{sec:dist-shift-exps}. Furthermore, we also demonstrate the
robustness of our algorithm to minor distribution shifts between
public unlabeled and private labeled data. In~\Cref{sec:exp-low-data},
we show our algorithm is also robust to both
small-sized private labeled datasets and public unlabeled datasets.

\subsection{Effectiveness under Distribution Shift}
\label{sec:dist-shift-exps}

\paragraph{Distribution Shift between Pre-Training and Private
Data} We consider private datasets that exhibit varying levels of
dissimilarity compared to the ImageNet pre-training dataset:
Flower-16~\citep{Flowers16}, GTSRB~\citep{GTSRB},
Pneumonia~\citep{kermany2018identifying}, a fraction ($12.5\%$) of
PCAM~\citep{PCAM}, and DermNet~\citep{DermNet}.
In~\Cref{fig:alldatasets}, we provide visual samples from each of
these datasets. Flower-16 and GTSRB have minimal overlap with
ImageNet-1K, with only one class in Flower-16 and 43 traffic signs
aggregated into a single label in ImageNet-1K. The Pneumonia, PCAM,
and DermNet datasets do not share any classes with ImageNet-1K. We
also observe that, given a fixed pre-training distribution and model,
different training procedures can have a different impact in the
utility of the extracted features for each downstream classification
task. Therefore, for each dataset we report the best performance
produced by the most useful pre-training algorithm. Results for all
the 5 pre-training strategies we consider and a discussion of how to
choose them is relegated to \Cref{sec:add_exp}. 
\begin{table}[t]\centering
    \begin{subtable}[b]{0.51\linewidth}
    \center 
           \vspace{-50px}

      \resizebox{\linewidth}{!}
{\small
    \begin{tabular}{l|cc|cc} 
            \toprule
            &\multicolumn{2}{c}{CIFAR10} & \multicolumn{2}{c}{CIFAR100} \\\hline
            \diagbox[width=3cm]{PCA Data}{Pre-training}            & SL & BYOL & SL & BYOL \\ \midrule
            In-distribution &  81.21  & 72.33 & 27.47 & 19.89 \\
            CIFAR-10v1 & 81.18 &  73.24& 27.18 & 19.21\\
            
            \bottomrule
            \end{tabular}
            }
            \caption{\footnotesize Distribution Shift between public (PCA) and private data}
            \label{tab:c10v1} 
    
     \end{subtable}\hfill
     \begin{subtable}[b]{0.48\linewidth}
    \centering
        \resizebox{\linewidth}{!}
{\normalsize
    \begin{tabular}{c|cc|cc} 
    \toprule
    & \multicolumn{2}{c}{CIFAR10} & \multicolumn{2}{c}{GTSRB} \\\hline
    \diagbox[width=4.5cm]{PCA Data}{Pre-training}  & SL & BYOL & SL & BYOL \\ 
    \hline
    1\% & 79.93 & 72.27 & 45.59 & 35.91\\
    5\% & 81.02 & 72.33  & 45.64 & 35.88\\
    10\% & 81.21  & 72.33 & 46.12 & 35.97 \\
    \bottomrule
    \end{tabular}}
    \caption{\footnotesize Varying amounts of public (PCA) data}
    \label{tab:low-public-data}
         \end{subtable}
\caption{\small a): Comparison between using the same amount of in-distribution data (i.e.
    10\% of CIFAR-10 and CIFAR-100 respectively) and  CIFAR-10v1 
    for computing the PCA projection~$(\epsilon=0.1)$. b)~Performance of~\ouralgo with varying amounts of public~(in distribution) data for computing the PCA projection~(\(\epsilon=0.1\)). The amount of public data is presented as a fraction of the whole available dataset.}
\label{fig:difficult}
\end{table}
\begin{wrapfigure}{r}{0.3\linewidth}
    \vspace{-10pt}
    \includegraphics[width=0.99\linewidth]{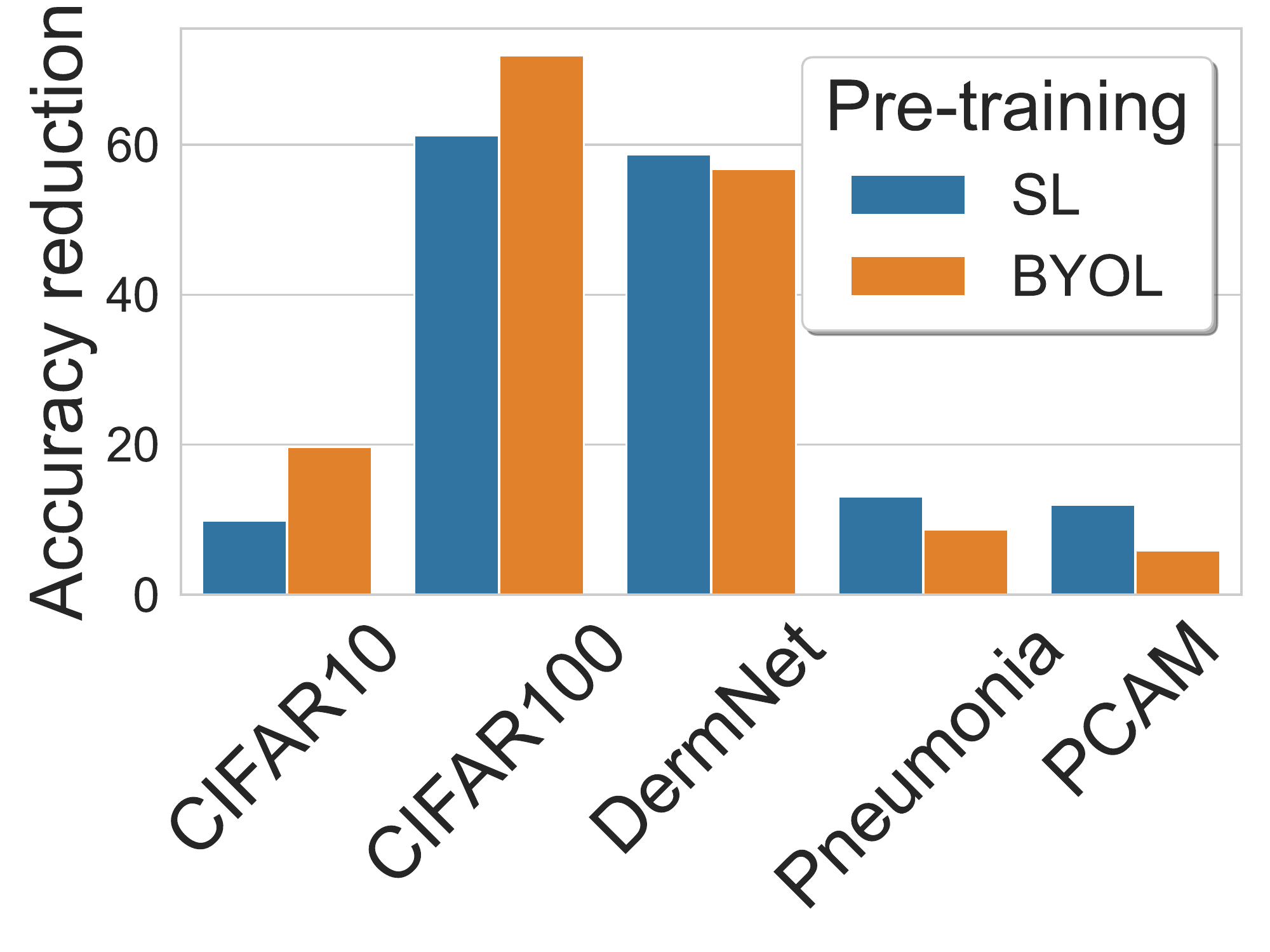}
    \caption{\small Comparing reduction in test accuracy for different datasets between using SL and BYOL pre-trained networks.}
    \label{fig:sslvssl}\vspace{-20pt}
\end{wrapfigure} 
In~\Cref{fig:farshift}, we demonstrate that reducing the
dimensionality of the pre-trained models enhances differentially
private training, irrespective of the private dataset used. For
instance, when applying the SL feature extractor to Flower-16 at at
$\epsilon=0.1$ with $k=40$, the accuracy improves to \(69.3\%\)
compared to only \(41.2\%\) when using the full dimensionality. For
DermNet, PCAM, and Pneumonia at \(\epsilon=0.1\), we observe accuracy
improvements from \(14\%\) to \(19\%\), from \(75\%\) to \(77.5\%\),
and from \(65\%\) to \(83\%\) respectively. Dimensionality reduction
has a more pronounced effect on performance when tighter privacy
constraints are imposed. It is worth noting that using dimensionality
reduction can significantly degrade performance for non-DP training,
as observed in CIFAR-10 and CIFAR-100.

 \paragraph{When to use labels in pre-training} We also
investigate the impact of different pre-training strategies on DP test
accuracy. From~\Cref{fig:neardist,fig:farshift}, we observe that some
pre-trained models are more effective than others for specific
datasets. To measure the maximum attainable accuracy with a publicly
trained classifier, we compute the drop in performance, observed by
training a DP classifier on BYOL pre-trained features, and the drop in
performance for SL pre-trained features. We then plot the fractional
reduction for both BYOL and SL across all the datasets for
\(\epsilon=0.1\) in Figure \ref{fig:sslvssl}.  In
\Cref{fig:semiSLvsSSL} we compare the relative reduction in
performance when using Semi-supervised pre-training and BYOL
pre-training. We find that datasets with daily-life objects and
semantic overlap with ImageNet-1K benefit more from leveraging SL
features and thus have a smaller reduction in accuracy for SL features
compared to BYOL features. In contrast, datasets with little label
overlap with ImageNet-1K benefit more from BYOL features, consistent
with findings by \citet{shi2022how}.

\paragraph{Distribution Shift between $\Sunl$ and $\Slab$} We
demonstrate the effectiveness of our algorithm even when the public
unlabeled data (used for computing the PCA projection matrix) is
sourced from a slightly different distribution than the private
labeled dataset. Specifically, we utilize the
CIFAR-10v1~\citep{Cifar10v1} dataset and present the results
in~\Cref{tab:c10v1}. Notably, CIFAR-10v1 consists of only \(2000\)
samples (\(4\%\) of the training data), yet the results for both
CIFAR-10 and CIFAR-100 remain essentially unchanged.  This finding
indicates that the data used to compute the PCA projection matrix does
not necessarily have to originate from the same distribution as the
private data and underscores that large amounts of public data are not
required for our method to be effective.
\subsection{Effectiveness in Low-Data Regimes}
\label{sec:exp-low-data}

\begin{figure}[t]
    \centering
    \includegraphics[width=0.99\linewidth]{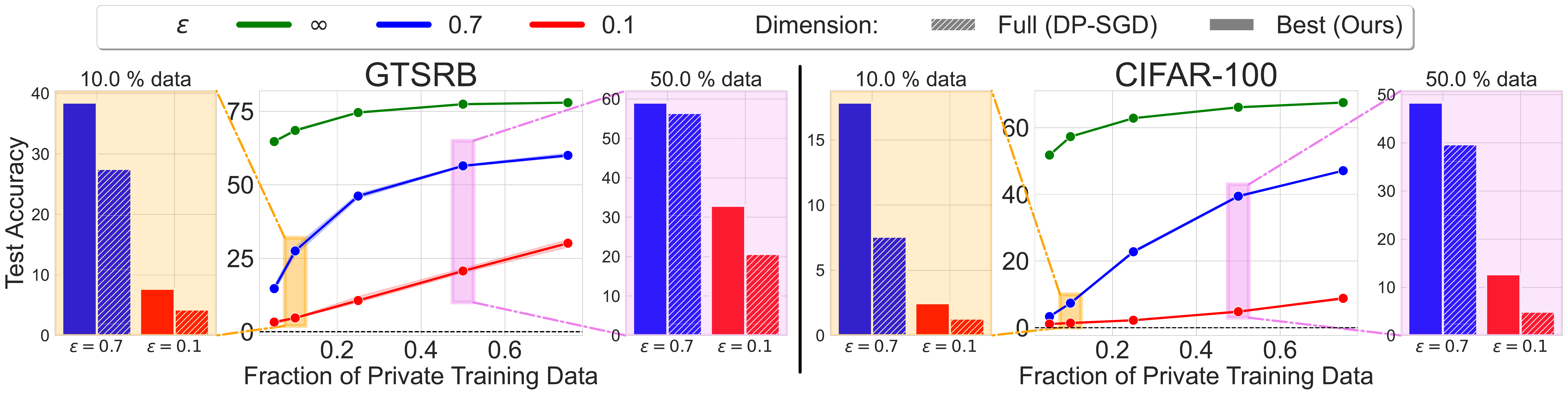}
    \caption{\small Middle columns: Test accuracy as the percentage of
    available training data varies in $\{0.05,0.1,0.25,0.5,0.75\}$.
    First and third columns: increase in test accuracy by applying
    PCA~(solid bars) vs not reducing dimension~(dashed bars) in
    low-data regimes for 50\%~(violet) and 10\%~(orange) private
    data.}
    \label{fig:sample_complexity_pca}
    \vspace{-10pt}
    \end{figure}

In privacy-critical settings such as medical contexts, there is often
a limited availability of training data. For instance, the DermNet and
pneumonia datasets contain only 12,000 and 3,400 training data points,
respectively, which is significantly smaller compared to datasets like
CIFAR-10 with 50,000 samples. To examine the impact of reduced data
(both private labeled and public unlabeled) on privacy, in this
section we conduct ablations using different fractions of training
data on the CIFAR-100 and GTSRB datasets. Please refer
to~\Cref{app:hyperparam-cifar10-100} for more details.  . 

\paragraph{Less private labelled data}
In~\Cref{fig:sample_complexity_pca}, we present the performance of
private and public training using different percentages of labeled training data for CIFAR-100 and GTSRB. Our results indicate
that under stringent privacy constraints~\(\br{\epsilon
\in\bc{0,7,0.1}}\), the performance of DP training, without
dimensionality reduction, is considerably low.  Conversely, even with
a small percentage of training data, non-DP training
demonstrates relatively high performance. By applying our algorithm in
this scenario, we achieve significant performance improvements
compared to using the full-dimensional embeddings. For instance,
applying PCA with with \(k=40\) dimensions enhances the accuracy of
our proposed algorithm from \(7.53\%\) to \(18.3\%\) on \(10\%\) of
CIFAR-100, with \(\epsilon=0.7\) using the SL feature extractor.
Similar improvements are also shown for GTSRB and to a smaller extent
for \(\epsilon=0.1\)

\paragraph{Less public unlabelled data} We demonstrate the robustness
of our algorithm to reduced amounts of public unlabeled data.
In~\Cref{tab:low-public-data}, we vary the amount of public data and
observe that the performance of our algorithm remains mostly stable
despite decreasing amounts of public data.

\section{Conclusion}
In this paper, we consider the setting of semi-private learning where
the learner has access to public unlabelled data in addition to
private labelled data. This is a realistic setting in many
circumstances e.g. where some people choose to make their data public.
Under this setting, we proposed a new algorithm to learn linear
halfspaces. Our algorithm uses a mix of PCA on unlabelled data and DP
training on private data. Under reasonable theoretical assumptions, we
have shown the proposed algorithm is $(\epsilon,\delta)$-DP and
provably reduces the sample complexity. In practical applications, we
performed an extensive set of experiments that show the proposed
technique is effective when tight privacy constraints are imposed,
even in low-data regimes and with a significant distribution shift
between the pre-training and private distribution.

\section{Acknowledgements}
AS acknowledges partial support from the ETH AI Center Postdoctoral
fellowship. FP acknowledges partial support from the European Space
Agency. AS, FP, and YH also received funding from the Hasler
Stiftung. We also thank Florian Tramer for useful feedback.

\bibliography{dp_ssl_ref}
\bibliographystyle{unsrtnat}

\newpage
\appendix
\onecolumn
\section{Proofs}
\label{sec:appendix-proof}
\subsection{Noisy SGD}
    In this section, we present~\cref{alg:noisy-sgd}, an adapted version of the Noisy SGD algorithm from \cite{BassilyST14} for $d$-dimensional linear halfspaces $\linear{d}$, that is used as a sub-procedure in \cref{alg:no_shift}.~\cref{alg:noisy-sgd} first applies a base procedure $\cA_{Base}$ on $\linear{d}$ for $k$ times to generate a set of $k$ results while preserving $(\epsilon, \delta)$-DP, and then applies the exponential mechanism $\cM_E$ to output one final result from the set.

\begin{algorithm}
  \caption{$\cA_{Noisy-SGD}( \Slab, \ell, (\epsilon, \delta), \beta)$}
  \label{alg:noisy-sgd}
  \begin{algorithmic}[1]
    \Procedure{$\cA_{Noisy-SGD}$}{$\Slab, \ell, (\epsilon, \delta), \beta$}
    \State \textbf{Input:} a labelled dataset $\Slab$, a loss function $\ell$, privacy parameters $\epsilon, \delta$, and the failure probability $\beta$. 
    \State Set $k = \lceil\log\nicefrac{1}{\beta}\rceil$. 
    \For{$i = 1$ to $k$}
    \State $\hw^{(i)} \leftarrow \cA_{Base}(\Slab, \ell, (\nicefrac{\epsilon}{k}, \nicefrac{\delta}{k}))$
    \EndFor
    \State Let $\cO \leftarrow \{\hw^{(1)}, \ldots, \hw^{(k)}\}$. 
    \State $\hw\leftarrow \cM_E(\Slab, -\ell, \cO,  \epsilon)$. 
    \State \textbf{Output:} $\hw$
    \EndProcedure 
    \Procedure{$\cA_{Base}$}{$( \Slab, \ell, (\epsilon, \delta))$} 
    \State \textbf{Input:} a labelled dataset $\Slab$, a loss function $\ell$, privacy parameters $\epsilon, \delta$.
    \State Let $\cL$ be the Lipschitz coefficient of the loss function $\ell$ and $\nlab$ be the size of $\Slab$.
    \State Set noise variance $\sigma^2 \leftarrow \frac{32L^2\br{\nlab}^2 \log (\nlab/\delta)\log(1/\delta)}{\epsilon}$.
    \State Randomize $\hat{w}^0\in \linear{d}$. 
    \State Set the learning rate function $\eta(t) = \frac{1}{\sqrt{t(\nlab)^2\cL^2 + m\sigma^2}}$. 
    \For{$t = 1$ to $(\nlab)^2 -1$}
    \State Uniformly choose $(x, y)\in \Slab$.
    \State Update $\hat{w}^{t+1} = \Pi_{\cW}\br{\hat{w}_t - \eta(t) [\nlab\nabla\ell(\hat{w}^t; (x, y)) + \xi]}$ where $\xi \sim N(0, \mathbb{I}_d \sigma^2)$.
    \EndFor 
    \State \textbf{Output:} $\hat{w} = \hat{w}^{(\nlab)^2}$
    \EndProcedure
    \Procedure{$\cM_E$}{$\Slab, \ell, \cO, \epsilon$} 
    \State \textbf{Input:} a dataset $\Slab$, a loss function $\ell$, an set of parameters $\cO$, and a privacy parameter $\epsilon$.
    \State Set the global sensitivity as $\Delta_U = \max_{S, S'} \max_{w\in \cO} |\ell(S, w) - \ell(S', w)|$, for any $S, S'$ of size $|\Slab|$ differing at exactly one entry. 
    \State \textbf{Output:} $w\in \cO$ with probability proportional to $\exp\br{\frac{\epsilon\ell(\Slab, w)}{2\Delta_U}}$.
    \EndProcedure
\end{algorithmic}
\end{algorithm}

In~\Cref{lem:noisy-sgd-guarantee-emp}, we state the privacy guarantee and the high probability upper bound on the excess error of the adpated version of Noisy SGD (\cref{alg:noisy-sgd}). This is a corollary of Theorem 2.4 in \cite{bassily2014private}, which provides an upper bound on the expected excess risk of $\cA_{Base}$. The proof of~\Cref{lem:noisy-sgd-guarantee-emp} follows directly from Markov inequality and the post-processing property of DP (\Cref{lem:dp-post-processing}), as described in Appendix D of \cite{bassily2014private}. 

  \begin{lem}[Theoretical guarantees of Noisy SGD \citep{BassilyST14}]
    \label{lem:noisy-sgd-guarantee-emp}
    Let the loss function $\ell$ be $\cL$-Lipschitz and $\linear{d}$ be the $d$-dimensional linear halfspace with diameter $1$. Then $\cA_{Noisy-SGD}$ is $(\epsilon, \delta)$-DP, and with probability $1-\beta$, its output $\hw$ satisfies the following upper bound on the excess risk, 
    \[\sum_{(x, y)\in S}\ell(\hat{w}, (x, y)) - \sum_{(x, y)\in S}\ell(\tw, (x, y)) = \frac{\cL\sqrt{d}}{\epsilon}\cdot\text{polylog}(n, \frac{1}{\beta}, \frac{1}{\delta}),\]
    for a labelled dataset $S$ of size $n$. Here, $\tw$ is the empirical risk minimizer $\tw= \text{argmin}_{w\in \cC}\sum_{(x, y)\in S}\ell(w, (x, y))$.  
  \end{lem}
\subsection{Proof of \Cref{thm:no-shift}}
\label{sec:appendix-no-shift}
\noshift* 
\begin{proof}
    \textbf{Privacy guarantee} Algorithm $\cA_{\epsilon, \delta}(k,
    \zeta)$ computes the transformation matrix $\hA$ on the public
    unlabelled dataset. This step is independent of the labelled data
    $\Slab$ and has no impact on the privacy with respect to $\Slab$.
    \(\Anoisy\) ensures the operations on the labelled dataset $\Slab$
    to output $v_k$ is $(\epsilon, \delta)$-DP with respect to
    $\Slab$~(\Cref{lem:noisy-sgd-guarantee-emp}). The final output
    $\hw = \hA v_k$ is attained by post-processing of $v_k$ and
    preserves the privacy with respect to $\Slab$ by the
    post-processing property of differential privacy
    (\Cref{lem:dp-post-processing}). 
    \begin{lem}[Post-processing \citep{dwork06dp}]
        \label{lem:dp-post-processing}
        For every $(\epsilon, \delta)$-DP algorithm $\cA: S \rightarrow \cY$ and every (possibly random) function $f: \cY \rightarrow \cY'$, $f\circ \cA$ is $(\epsilon, \delta)$-DP. 
    \end{lem}

    \textbf{Accuracy guarantee} By definition, all distributions
    $D_{\gamma, \xi_k}\in \cD_{\gamma_0, \xi_0}$ are $(\gamma,
    \xi_k)$-large margin low rank for some $\gamma \geq\gamma_0, \xi_k
    \leq \xi_0$. Let the empirical covariance matrix of $D_{\gamma,
    \xi_k}$ calculated with the unlabelled dataset $\Sunl$ be $\ecov =
    \frac{1}{\nunl} \sum_{x\in \Sunl}(x - \bar{x})(x -
    \bar{x})^{\top}$ and \(\hA\in\reals^{d\times k}\) be the
    projection matrix whose \(i^{\it th}\) column is the \(i^{\it
    th}\) eigenvector of \(\ecov\).  Let \(\cov\) be the population
    covariance matrix and similarly, let \(A_k\) the matrix of top
    \(k\) eigenvectors of \(\cov\).
    
    For any distribution $D_{\gamma, \xi_k}\in \cD_{\gamma_0, \xi_0}$,
    let $D_{X, \br{\gamma, \xi_k}}$ be the marginal distribution of
    $X$ and \(\tw\) be the large margin linear classifier that is
    guaranteed to exist by~\Cref{defn:low-dim-large-margin}. The
    margin after projection by $\hA$ is lower bounded by $
    \frac{y\ba{\hA^{\top}z,
    \hA^{\top}\tw}}{\norm{\hA^Tz}_2\norm{\hA^{\top} \tw}_2}$ for any
    $z\in \supp{D_{X, \br{\gamma, \xi_k}}}$.
    
    We will first derive a high-probability lower bound for this term
    to show that, after projection, with high probability, the
    projected distribution still has a large margin. Then, we will
    invoke existing algorithms in the literature with the right
    parameters, to privately learn a large margin classifier in this
    low-dimensional space.
    
    Let $z$ be any vector in $ \supp{D_{X, \br{\gamma, \xi_k}}}$. We
    can write $z = a_z\tw + b^{\perp}$ for some \(a_z\) where
    $b^{\perp}$ is in the nullspace of $\tw$. Then, it is easy to see
    that using the large-margin property in
    \Cref{defn:low-dim-large-margin}, we get\begin{equation}
        \label{eq:compatibility_shift_1}
                ya_z = \frac{\ba{\tw, z}}{\norm{\tw}_2\norm{z}_2}\geq \gamma\geq \gamma_0.
    \end{equation}
    Then, we lower bound $\dfrac{y\ba{\hA^{\top}z,
    \hA^{\top}\tw}}{\norm{\hA^{\top}z}_2\norm{\hA^{\top} \tw}_2}$ as
    \begin{equation}
        \label{eq:shift_margin_der_1}
        \begin{aligned}
            \frac{y\ba{\hA^{\top}z,
            \hA^{\top}\tw}}{\norm{\hA^{\top}z}_2\norm{\hA^{\top} \tw}_2} &\overset{(a)}{=}  \frac{ya_z\norm{\hA^{\top} \tw}_2^2}{\norm{\hA^{\top}z}_2\norm{\hA^{\top} \tw}_2} 
            \overset{(b)}{\geq} \gamma_0\norm{\hA^{\top}\tw}_2
        \end{aligned}
        \end{equation}
    where step $(a)$ is due to $\ba{\tw, b^{\perp}} = 0$ and step
    $(b)$ follows from $\norm{\hAu^{\top} z}_2 \leq
    \norm{\hAu}_{\op}\norm{z}_2\leq 1$
    and~\Cref{eq:compatibility_shift_1}.  
    
    To lower bound \(\norm{\hA^ \top\tw}_2\), note that
    \begin{equation}
        \label{eq:hA_rip1}
        \begin{aligned}
            \norm{\hA^{\top}\tw}_2= \norm{\hA\hA^{\top} \tw}_2 &\geq \norm{\A\A^{\top} \tw}_2 - \norm{\hA\hA^{\top}\tw - \A\A^{\top}\tw} _2&&\text{by Triangle Inequality}\\
             &\geq \norm{\A\A^{\top} \tw}_2  - \norm{\hA\hA^{\top} - \A\A^{\top}}_F\norm{\tw}_2 &&\text{by Cauchy Schwarz Inequality}\\
             &\geq 1-\xi_k- \norm{\hA\hA^{\top} - \A\A^{\top}}_F.
        \end{aligned}
        \end{equation} where the last step follows from the low rank assumption in~\Cref{defn:low-dim-large-margin} and observing that \(\norm{\tw}_2 = 1\).

        To upper bound \(\norm{\hA\hA^{\top} - \A\A^{\top}}_F\), we
        use~\Cref{lem:pca_guarantee}. 
        \begin{lem}[Theorem 4 in \cite{ZwaldB05}]
        \label{lem:pca_guarantee}
        Let $D$ be a distribution over $\{x\in
        \reals^d\vert~\norm{x}^2\leq 1\}$ with covariance matrix
        $\cov$ and zero mean $\bE_{x\sim D}[x] = 0$. For a sample $S$
        of size $n$ from $D$, let $\ecov = \frac{1}{n}\sum_{x\in
        S}xx^T$ be the empirical covariance matrix. Let $A_k, \hA$ be
        the matrices whose columns are the first $k$ eigenvectors of
        $\cov$ and $\ecov$ respectively and let
        $\eigval{1}{\cov}>\eigval{2}{\cov}>\ldots>\eigval{d}{\cov}$ be
        the ordered eigenvalues of $\cov$. For any
        $k>0,\beta\in\br{0,1}$ such that $\eigval{k}{\cov}>0$ and
        $n\geq \frac{16\br{1+\sqrt{\beta/2}}^2}{\br{\eigval{k}{\cov} -
        \eigval{k+1}{\cov}}^2}$, we have that with probability at
        least $1-e^{-\beta}$,
        \begin{equation*}
            \norm{\A\A^T - \hA\hA^T}_F\leq \frac{4\br{1+\sqrt{\frac{\beta}{2}}}}{\br{\eigval{k}{\cov}-\eigval{k+1}{\cov}}\sqrt{n}}.
        \end{equation*}
        \end{lem}
    
        It guarantees that with
        probability $1-\frac{\beta}{2}$,
        \begin{equation}
            \label{eq:covariance_guarantee}
            \begin{aligned}
            \norm{\A\A^{\top} - \hA\hA^{\top}}_F
            \leq \frac{4\br{1 +
            \sqrt{\frac{\log\br{\nicefrac{2}{\beta}}}{2}}}}{\br{\eigval{k}{\cov}-\eigval{k+1}{\cov}}\sqrt{\nunl}}
            \leq \frac{\gamma_0}{10}. 
            \end{aligned}
        \end{equation}
        where the last inequality follows from choosing the size of
    unlabelled data \(\nunl \geq \frac{1600\br{1 +
    \sqrt{\frac{\log\br{\nicefrac{2}{\beta}}}{2}}}^2}{\gamma_0^2\br{\Delta_{\min}\lambda_k}^2}\).

    Substituting~\Cref{eq:covariance_guarantee}
    into~\Cref{eq:hA_rip1}, we get that with probability $1-\frac{\beta}{2}$, 
    \begin{equation}\label{eq:hAtw_lower_bound}
        \norm{\hA^\top\tw}_2\geq 1-\xi_k-\frac{\gamma}{10}
    \end{equation}
    
    Plugging \cref{eq:hAtw_lower_bound} into \cref{eq:shift_margin_der_1}, we derive a high-probability lower bound on the distance of any point to the decision boundary
    in the transformed space. For all $z\in \supp{D_{X, \br{\gamma, \xi_k}}}$,
    \begin{equation*}
        \label{eq:final_margin_y=1}
        \frac{y\ba{\hA^{\top}z, \hA^{\top}\tw}}{\norm{\hA^{\top}z}_2\norm{\hA^{\top} \tw}_2} \geq \gamma_0\br{1-\xi_0-\frac{\gamma_0}{10}}.
    \end{equation*}

    This implies that the margin in the transformed space is at least
    $\gamma_0\br{1-\xi_0-\nicefrac{\gamma_0}{10}}$. Thus, the~(scaled)
    hinge loss function $\ell$ defined in \cref{eq:loss-funtion-alg}
    is $\frac{1}{\gamma_0\br{1-\xi_0
    -\nicefrac{\gamma_0}{10}}}$-Lipschitz. For a halfspace with
    parameter $v \in B_2^k$, denote the empirical hinge loss on a
    dataset $S$ by $\hat{L}(w; S) = \frac{1}{\abs{S}}\sum_{(x, y)\in
    S}\ell(w, (x, y))$ and the loss on the distribution $D$ by $L(w;
    D) = \bP_{(x, y)\sim D}\bs{\ell(w, (x, y))}$. Let $D_k$ be the
    $k$-dimension transformation of the original distribution $D$
    obtained by projecting each $x\in \cX$ to $\hA^\top x$. By the
    convergence bound in~\Cref{lem:noisy-sgd-guarantee-emp}
    for~\(\Anoisy\), we have with probability
    $1-\frac{\beta}{4}$,~\(\Anoisy\) outputs a hypothesis $v_k\in
    B_2^k$ such that 
    \begin{equation*}
        \hat{L}(v_k; \Slab_k) - \hat{L}(\vstar; D_k) = \hat{L}(v_k; \Slab_k) = \tilde{O}\br{\frac{\sqrt{k}}{\nlab\epsilon\gamma_0\br{1-\xi_0-0.1\gamma_0}}}
    \end{equation*}
    where $\vstar = \text{argmin}_{v\in B_2^k} \hat{L}(v; \Slab_k)$ and
    $\hat{L}(\vstar; \Slab_k) = 0$ as the margin in the transformed space is at least $\gamma_0\br{1-\xi_0 - \frac{\gamma_0}{10}}\geq 0$. 
    
    Then, we can upper bound the empirical 0-1 error by the
    empirical~(scaled) hinge loss in the $k$-dimensional transformed space. For $\nlab =
    \br{\frac{\sqrt{k}}{\alpha\epsilon\gamma_0(1-\xi_0-0.1\gamma_0)}}\polylog{\frac{1}{\delta},\frac{1}{\epsilon},\frac{1}{\beta},\frac{1}{\alpha},\frac{1}{\gamma_0},\frac{1}{\xi_0},k, 
    \nlab}$,
    with probability $1-\frac{\beta}{4}$, 
    \begin{equation}
        \label{eq:no-shift-empirical-guarantee}
       \frac{1}{\nlab} \sum_{(x, y)\in \Slab}\mathbbm{I}\{y\ba{x, v_k}< 0\} \leq \hat{L}(v_k; \Slab_k) =\tilde{O}\br{\frac{\sqrt{k}}{\nlab\epsilon\gamma_0\br{1-\xi_0-0.1\gamma_0}}}\leq \frac{\alpha}{4}
    \end{equation}

    It remains to bound the generalisation error of linear halfspaces
    $\linear{k}$. That is, we still need to show that the empirical
    error of a linear halfspace is a good approximation of the error
    on the distribution. To achieve this, we can
    apply~\Cref{lem:generalization-bound} to upper bound the
    generalisation error using the growth function. 
        \begin{lem}[Convergence bound on generalisation error \citep{anthony_bartlett_1999}]
        \label{lem:generalization-bound}
        Suppose $\cH$ is a hypothesis class with instance space $\cX$ and output space $\{-1, 1\}$. Let $D$ be a distribution over $\cX\times \cY$ and $S$ be a dataset of size $n$ sampled i.i.d. from $D$. For $\eta \in (0, 1), \zeta > 0$, we have \[\bP_{S\sim D^n}\bs{\sup_{h\in \cH}L(h; D) - (1+\zeta)\hat{L}(h; S)> \eta} \leq 4\Pi_{\cH}(2n)\exp\br{-\frac{\eta\zeta n}{4\br{\zeta + 1}}},\]
        where $L$ and $\hat{L}$ are the population and the empirical 0-1 error and $\Pi_{\cH}$ is the growth function of $\cH$. 
      \end{lem}
    
    As the Vapnik-Chervonenkis
    (VC) dimension of $k$-dimensional linear halfspaces is $k+1$, its
    growth function is bounded by $\Pi(2\nlab)\leq (2\nlab)^{k + 1} +
    1$. Setting~\(\eta=\frac{\alpha}{2},\zeta=1\)
    in~\Cref{lem:generalization-bound} and combining
    with~\Cref{eq:no-shift-empirical-guarantee}, we have that with
    probability $1-\frac{\beta}{4}$, we have 
    \[\bP_{(x, y)\sim D}\bs{y\ba{v_k, \hA^\top x}<0} = \bP_{(x, y)\sim
    D_k}\bs{y\ba{v_k, x}<0} \leq \frac{2}{\nlab} \sum_{(x, y)\in
    \Slab_k}\mathbbm{I}\{y\ba{v_k,x}\leq 0\} + \frac{\alpha}{2} \leq
    \alpha, \] for \(\nlab \geq
    \frac{k}{\alpha}\mathrm{polylog}\br{\frac{1}{\beta},\frac{1}{k}}\).
    This is equivalent as stating that the output
    of~\Cref{alg:no_shift} $\hw = \hA v_k$ satisfies 
    \[\bP_{(x, y)\sim D}\bs{y\ba{\hw, x}<0}\overset{(a)}{=} \bP_{(x, y)\sim D}\bs{y\ba{\hA v_k, x}<0} \overset{(b)}{=} \bP_{(x, y)\sim D}\bs{y\ba{v_k, \hA^\top x}<0} \leq \alpha,\]
    where Equality $(a)$ follows from the definition of $\hw = \hA v_k$ and Equality $(b)$ follows from the definition of $k$-dimensional transformation $D_k$ of a distribution $D$. This concludes the proof. 
    \end{proof}

    \subsection{Privacy guarantees for~\ouralgo~on the original image dataset}
    As described in~\Cref{fig:diagram-our-algo}, in
    practice~\ouralgo~is applied on the set of representations
    obtained by passing the private dataset of images through a
    pre-trained feature extractor.  Therefore, a straightforward
    application of~\Cref{thm:no-shift} yields
    an~\(\br{\epsilon,\delta}\)-DP guarantee on the set of
    representations and not on the dataset in the raw pixel space
    themselves. Here, we show that~\ouralgo~provides~(at least) the
    same~DP guarantees on the dataset in the pixel space as long as
    the pre-training dataset cannot be manipulated by the privacy
    adversary. One way to achieve this, as we show is possible in this
    paper, is by using the same pre-trained model across different
    tasks. Investigating the extent of privacy harm that can be caused
    by allowing the adversary to manipulate the pre-training data
    remains an important future direction.

    \begin{corollary}\label{prop:dp-orig-data} Let $f:\reals^p\to\reals^d$ be a feature extractor pre-trained using any algorithm. Let
        \(\Zlab_1, \Zlab_2\) be any two neighbouring datasets of
        private images in \(\reals^p\). Then, for any \(Q\subseteq\linear{d}\) where \(\linear{d}\) is the class of linear halfspaces in \(d\) dimensions,
         \[\bP_{h\sim \cA_{\epsilon,\delta}\circ
         f\br{\Zlab_1}}\bs{h\in Q} \leq
         e^{\epsilon}\bP_{h\sim\cA_{\epsilon,\delta}\circ
         f\br{\Zlab_2}}\bs{h\in Q} + \delta\]
         where \(\cA_{\epsilon,\delta}\) is~\Cref{alg:no_shift}~(\ouralgo) run with privacy parameters \(\epsilon,\delta\).
    \end{corollary}

    \begin{proof}
    Note that $f$ is a deterministic many-to-one function from the
   dataset of images to the dataset of representations \footnote{$f$
   can be designed to normalize the extracted features in a
   $d$-dimensional unit ball.}. For any two neighbouring datasets
   \(\Zlab_1, \Zlab_2\) in the image space, let $\Rlab_1, \Rlab_2$ be
   the corresponding set of representations extracted by $f$,
   \ie~$\Rlab_1 = \bc{f(x): x\in \Zlab_1}$ and $\Rlab_2 = \bc{f(x):
   x\in \Zlab_2}$.  Then for any \(Q\subseteq \linear{d}\)
    
    \begin{equation*}
        \begin{aligned}
            \bP_{h\sim \cA_{\epsilon,\delta}\circ f \br{\Zlab_1}}\bs{h\in Q}&= \bP_{h\sim\cA_{\epsilon,\delta}\br{\Rlab_1}}\bs{h\in Q}\\
            &\leq e^{\epsilon}\bP_{h\sim\cA_{\epsilon,\delta}\br{\Rlab_2}}[h\in Q] + \delta\\
            &= e^{\epsilon}\bP_{h\sim\cA_{\epsilon,\delta}\circ f\br{\Zlab_2}}\bs{h\in Q} + \delta
        \end{aligned}
    \end{equation*}

    where the first and the last equality follows by using the
    definition \(\Rlab_1,\Rlab_2\) and due to the fact that \(f\) is a
    many-to-one function. The second inequality follows from observing
    that \(\Rlab_1, \Rlab_2\) can differ on at most one point as \(f\)
    is a deterministic many-to-one function
    and~\(\cA_{\epsilon,\delta}\) is \(\br{\epsilon,\delta}\)-DP.
\end{proof}

\subsection{Proof of~\Cref{thm:with-shift}}
\label{sec:appendix-with-shift}
In this section, we restate~\Cref{thm:with-shift} with additional
details and present its proof. Before that, we formally define
\(\eta\)-TV tolerant semi-private learning.

\begin{defn}[$\eta$-TV tolerant $(\alpha, \beta, \epsilon,
    \delta)$-semi-private learner on a family of distributions
    $\cD$]\label{defn:private-learning} An algorithm $\cA$ is an
    $\eta$-TV tolerant $(\alpha, \beta, \epsilon,
    \delta)$-semi-private learner for a hypothesis class $\cH$ on a
    family of distributions $\cD$ if for any distribution $D^L\in
    \cD$, given a labelled dataset $\Slab$ of size $\nlab$ sampled
    i.i.d. from $D^L$ and an unlabelled dataset $\Sunl$ of size
    $\nunl$ sampled i.i.d. from any distribution $D^U$ with
    $\eta$-bounded TV distance from $D^L_X$ as well as third moment
    bounded by $\eta$, $\cA$ is $(\epsilon ,\delta)$-DP with respect
    to $\Slab$ and outputs a hypothesis $h$ satisfying 
    \[\bP[\bP_{(x, y)\sim D}\bs{h(x)\neq y} \leq \alpha] \geq
    1-\beta,\]where the outer probability is over the randomness of
    the samples and the intrinsic randomness of the algorithm. In
    addition, the sample complexity $\nlab$ and $\nunl$ must be
    polynomial in $\frac{1}{\alpha}$ and $\frac{1}{\beta}$, and
    $\nlab$ must also be polynomial in $\frac{1}{\epsilon}$ and
    $\frac{1}{\delta}$. 
\end{defn}

Now, we are ready to prove that~\ouralgo~is an $\eta$-TV tolerant
$(\alpha, \beta, \epsilon, \delta)$-semi-private learner when
instantiated with the right arguments.

    \begin{thm}\label{thm:with-shift-full} For $k\leq
        d\in\bN,~\gamma_0\in (0, 1),~\xi_0\in [0, 1)$ satisfying
        $\xi_0 < \nicefrac{1}{2} - \nicefrac{\gamma_0}{10} $, let
        $\cD_{\gamma_0, \xi_0}$ be the family of distributions
        consisting of all \(\br{\gamma,\xi_k}\)-large margin low rank
        distributions over \(\instspace_d\times\cY\) with $\gamma \geq
        \gamma_0$ and $\xi_k\leq \xi_0$ and third moment bounded by
        $\eta$. For any $\alpha \in \br{0, 1}, \beta \in \br{0,
        \nicefrac{1}{4}}$, $\epsilon\in \br{0,
        \nicefrac{1}{\sqrt{k}}}$,  $\delta\in (0, 1)$ and $\eta \in
        [0, \nicefrac{\deltak}{14})$, $\cA_{\epsilon, \delta}(k,
        \zeta)$, described by~\Cref{alg:no_shift}, is an $\eta$-TV
        tolerant $(\alpha, \beta, \epsilon, \delta)$-semi-private
        learner of the linear halfspace $\linear{d}$ on
        $\cD_{\gamma_0, \xi_0}$ with sample complexity  
        \begin{equation*}
          \nunl = \bigO{\frac{\log\frac{2}{\beta}}{\br{\gamma_0\deltak}^2}},
          \nlab = \tilde{O}\br{\frac{\sqrt{k}}{\alpha\epsilon\zeta}}
          \end{equation*}
        where $\deltak = \eigval{k}{\covl}-\eigval{k+1}{\covl}$ and $\zeta = \gamma_0(1-\xi_0 - 0.1\gamma_0 -\nicefrac{7\eta}{\deltak})$. 
      \end{thm}

\begin{proof}
    \textbf{Privacy Guarantee} A similar argument as the proof of the
    privacy guarantee in~\Cref{thm:no-shift} shows that Algorithm
    $\cA_{\epsilon, \delta}(k, \zeta)$ preserves $(\epsilon,
    \delta)$-DP on the labelled dataset $\Slab$. We now focus on the
    accuracy guarantee.

    \noindent\textbf{Accuracy Guarantee} For any unlabelled
    distribution $D^U$ with $\eta$-bouneded TV distance from the
    labelled distribution $D_{\gamma, \xi_k}^L$, let the empirical
    covariance matrix of the unlabelled dataset $S^U$ be $\ecovu =
    \frac{1}{\nunl} \sum_{x\in \Sunl}xx^{\top}$ and
    \(\hAu\in\reals^{d\times k}\) be the projection matrix whose
    \(i^{\it th}\) column is the \(i^{\it th}\) eigenvector of
    \(\ecovu\). Let \(\covl\) and $\covu$ be the population covariance
    matrices of the labelled and unlabelled distributions $D^L$ and
    $D^U$.  Similarly, let \(\Al\) and $\Au$ be the matrices of top
    \(k\) eigenvectors of \(\covl\) and $\covu$ respectively.

    By definition, all distributions $D_{\gamma, \xi_k}^L\in
    \cD_{\gamma_0, \xi_0}$ are $(\gamma, \xi_k)$-large margin low rank
    distribution, as defined in~\Cref{defn:low-dim-large-margin}, for
    some $\gamma\geq \gamma_0$, $\xi_k \leq \xi_0$.  Let \(\tw\) be
    the large margin linear classifier that is guaranteed to exist
    by~\Cref{defn:low-dim-large-margin}. Then, for all $z\in
    \supp{D_{X, (\gamma, \xi_0)}^{L}}$, where $D_{X, (\gamma,
    \xi_0)}^{L}$ is the marginal distribution of $D_{\gamma, \xi_k}$,
    its margin is lower bounded by $\frac{y\ba{\hAu^{\top}z,
    \hAu^{\top}\tw}}{\norm{\hAu^{\top}z}_2\norm{\hAu^{\top}\tw}_2}$.
    Similar to the proof of~\Cref{thm:no-shift}, we will first lower
    bound this term to show that, with high probability, the projected
    distribution still retains a large margin.  Then, we will invoke
    existing algorithms in the literature with the right parameters,
    to privately learn a large margin classifier in this low
    dimensional space. 

    First, let $z = a_z\tw + b^{\perp}$ for some \(a_z\) where
    $b^{\perp}$ is in the nullspace of $\tw$. Then, it is easy to see
    that using the large-margin property in
    \Cref{defn:low-dim-large-margin}, we get\begin{equation}
        \label{eq:compatibility_shift}
                ya_z = \frac{\ba{\tw, z}}{\norm{\tw}_2\norm{z}_2}\geq \gamma\geq \gamma_0.
    \end{equation}
    Then, we lower bound $\dfrac{y\ba{\hAu^{\top}z,
    \hAu^{\top}\tw}}{\norm{\hAu^{\top}z}_2\norm{\hAu^{\top}
    \tw}_2}$ as
    \begin{equation}
        \label{eq:shift_margin_der}
        \begin{aligned}
                \frac{ y\ba{\hAu^{\top} z, \hAu^{\top} \tw}}{\norm{\hAu^{\top} z}_2\norm{\hAu^{\top} \tw}_2} &\overset{(a)}{=} \frac{ya_z\norm{\hAu^{\top} \tw}_2}{\norm{\hAu^{\top} z}_2} \overset{(b)}{\geq}\gamma_0\norm{\hAu^{\top} \tw}_2,
        \end{aligned}
        \end{equation}
    where step $(a)$ is due to $\ba{\tw, b^{\perp}} = 0$ and step
    $(b)$ follows from $\norm{\hAu^{\top} z}_2 \leq
    \norm{\hAu}_{\op}\norm{z}_2\leq 1$ and~\Cref{eq:compatibility_shift}.  To lower bound $\norm{\hAu^{\top} \tw}_2$, we use the
    triangle inequality to decompose it as follows 
    \begin{equation}\label{eq:lower-bound-unlabelled-estimation}
        \begin{aligned}
            \norm{\hAu^{\top} \tw}_2 \geq& \norm{\Al\br{\Al}^{\top}\tw}_2 - \norm{\br{\Al \br{\Al}^{\top} - \Au\br{\Au}^{\top}}\tw}_2 - \norm{\br{\Au(\Au)^{\top}  - \hAu(\hAu)^{\top}  }\tw}_2\\
            \geq& \norm{\Al\br{\Al}^{\top}\tw}_2  - \norm{\Al (\Al)^{\top} - \Au\br{\Au}^{\top}}_{\op}\norm{\tw}_2 - \norm{\Au(\Au)^{\top}  - \hAu(\hAu)^{\top}  }_F\norm{\tw}_2\\
            \geq& 1-\xi_k  - \norm{\Al (\Al)^{\top} - \Au\br{\Au}^{\top}}_{\op} - \norm{\Au(\Au)^{\top}  - \hAu(\hAu)^{\top}  }_F
        \end{aligned}
    \end{equation}
    where the second inequality follows from applying Cauchy-Schwartz
    inequality on the second and third term and the third step follows
    from using the low rank separability assumption
    in~\Cref{defn:low-dim-large-margin} on the first term and
    observing that $\norm{\tw}_2 = 1$.
    
    Now, we need to bound the two terms \(\norm{\Al (\Al)^{\top} -
    \Au\br{\Au}^{\top}}_{\op}\),\(\norm{\Au(\Au)^{\top}  -
    \hAu(\hAu)^{\top}  }_F\). We bound the first term
    with~\Cref{lem:proj_diff}. 
    \begin{lem}
        [Theorem 3 in \cite{ZwaldB05}]\label{lem:proj_diff}
        Let $A\in \bR^{d}$ be a symmetric positive definite matrix with nonzero eigenvalues $\lambda_1 > \lambda_2 > \ldots >\lambda_d$.  
        Let $k >0$ be an integer such that $\lambda_k > 0$. Let $B\in \bR^{d}$ be another symmetric positive definite matrix such that $\norm{B}_F < \frac{1}{4}\br{\lambda_k - \lambda_{k+1}}$ and $A + B$ is still a positive definite matrix. Let $P_k(A), P_k(A+B)$ be the matrices whose columns consists of the first $k$ eigenvectors of $A, A+B$, then \[\norm{P_k(A)P_k(A)^T - P_k(A+B)P_k(A+B)^T}_F\leq \frac{2\norm{B}_F}{\lambda_k - \lambda_{k+1}}.\]
    \end{lem}

    It guarantees that with probability $1-\nicefrac{\beta}{4}$, 
    \begin{equation}
        \label{eq:proj_space_dist_shift_bound}
            \norm{\Al\br{\Al}^{\top} - \Au\br{\Au}^{\top} }_{\op}\leq \frac{2\norm{\covl - \covu}_{\op}}{\eigval{k}{\covl}- \eigval{k+1}{\covl}}
    \end{equation}
     Then, we bound the term $\norm{\covl - \covu}_{\op}$ with \Cref{lem:bounded-covariance-shift}.  
    \begin{lem}\label{lem:bounded-covariance-shift} Let $f$ and $g$ be
        the Probability Density Functions (PDFs) of two zero-mean
        distributions $F$ and $G$ over $\cX$ with covariance matrices
        $\cov_f$ and $\cov_g$ respectively. Assume the spectral norm
        of the third moments of both $F$ and $G$ are bounded by
        $\eta$. If the total variation between the two distributions
        is bounded by $\eta$,\ie~$TV(f, g) = \max_{A\subset
        \cX}\abs{f(A) - g(A)} \leq \eta$, then the discrepancy in the
        covariance matrices is bounded by $7\eta$, \ie $\norm{\cov_f -
        \cov_g}_{\op}\leq 7\eta$.
    \end{lem}

     By applying~\Cref{lem:bounded-covariance-shift} and the assumption of bounded total variation between the labelled and unlabelled distributions to~\Cref{eq:proj_space_dist_shift_bound}, we get 
     \begin{equation}
         \label{eq:proj_space_dist_shift_bound1}
         \norm{\Al\br{\Al}^{\top} - \Au\br{\Au}^{\top} }_{\op}
            \leq \frac{7\eta}{\eigval{k}{\covl}- \eigval{k+1}{\covl}} = \frac{7\eta}{\deltak},
     \end{equation}
    
    Similar to the proof for \Cref{thm:no-shift}, we upper bound the
    term $\norm{\Au(\Au)^{\top}  - \hAu(\hAu)^{\top}  }_F$
    using~\Cref{lem:pca_guarantee}, which guarantees that with
    probability $1-\nicefrac{\beta}{4}$, \begin{equation}
        \label{eq:estimation_error_unlabel_space}
            \norm{\Au\br{\Au}^{\top} - \hAu \hAu^{\top}}_F\leq \frac{\gamma_0}{10}, 
    \end{equation}
    where the inequality follows from choosing the size of unlabelled
    data $\nunl =
    \bigO{\frac{\log\frac{2}{\beta}}{\br{\gamma_0\deltak}^2}}$.
    
    Substituting~\Cref{eq:estimation_error_unlabel_space,eq:proj_space_dist_shift_bound1}
    into~\Cref{eq:lower-bound-unlabelled-estimation} and then plugging
    \cref{eq:lower-bound-unlabelled-estimation} into
    \cref{eq:shift_margin_der}, we get that with probability at least
    $1-\nicefrac{\beta}{2}$, the margin in the projected space is
    lower bounded as
    \begin{equation*}
        \label{eq:margin_lower_bound_shift_y=1}
        \frac{y \ba{\hAu^{\top} z, \hAu^{\top} \tw}}{\norm{\hAu^{\top} z}_2\norm{\hAu^{\top} \tw}_2} \geq \gamma_0\br{1-\xi_0 - \frac{7\eta}{\deltak} - \frac{\gamma_0}{10}}.
    \end{equation*}
  
    Thus, the (scaled) hinge loss function $\ell$ defined in
    \cref{eq:loss-funtion-alg} in~\Cref{alg:no_shift} is
    $\frac{1}{\gamma_0\br{1-\xi_0-\nicefrac{7\eta}{\deltak}-
    0.1\gamma_0}}$-Lipschitz. For a halfspace with parameter $v\in B_2^k$, denote the
    empirical hinge loss on a dataset $S$ by $\hat{L}(w; S) =
    \frac{1}{|S|}\sum_{(x, y)\in S}\ell(w, (x, y))$ and the loss on
    the distribution $D$ by $L(w; D) = \bE_{(x, y)\sim D}\bs{\ell(w,
    (x, y))}$. Let $D_k$ be the $k$-dimension transformation of the original distribution $D$ by projecting each $x\in \cX$ to $\hA^\top x$. By the convergence bound
    in~\Cref{lem:noisy-sgd-guarantee-emp} for \(\Anoisy\), we have
    with probability $1-\frac{\beta}{4}$, \(\Anoisy\) outputs
    a hypothesis $v_k \in B_2^k$ such that  
    \begin{equation*}
        \label{eq:shift-empirical-loss-bound}
        \hat{L}(v_k; \Slab_k) - \hat{L}(\vstar; D_k) = \hat{L}(v_k; \Slab_k) = \tilde{O}\br{\frac{\sqrt{k}}{\nlab\epsilon \gamma_0\br{1-\xi_0-\nicefrac{7\eta}{\deltak}- 0.1\gamma_0}}},
    \end{equation*}
    where $\vstar = \text{argmin}_{v\in B_2^k} \hat{L}(v; \Slab_k)$ and
    $\hat{L}(\vstar; \Slab_k) = 0$ as the margin in the transformed low-dimensional space is at least $\gamma_0(1-\xi_0 - \frac{7\eta}{\deltak} - \frac{\gamma_0}{10}>0$. 
    For $\nlab =
    O\br{\frac{\sqrt{k}}{\alpha\beta\gamma_0\br{1-\xi_0-\nicefrac{7\eta}{\deltak}-
    0.1\gamma_0}}\text{polylog}\br{\frac{1}{\delta},
    \frac{1}{\epsilon}, \frac{1}{\beta}, \frac{1}{\alpha},
    \frac{1}{\gamma_0}, \frac{1}{\xi_0}, k,\nlab}}$, we can bound the
    emiprical 0-1 error with probability $1-\frac{\beta}{4}$, 
    \begin{equation}
        \label{eq:shift-empirical-guarantee}
       \frac{1}{\nlab} \sum_{(x, y)\in \Slab_k}\mathbbm{I}\bc{y\ba{v_k, x}< 0} \leq \hat{L}(v_k; \Slab_k) =\tilde{O}\br{\frac{\sqrt{k}}{\nlab\epsilon\gamma_0\br{1-\xi_0-\nicefrac{7\eta}{\deltak}- 0.1\gamma_0}}}\leq \frac{\alpha}{4}.
    \end{equation}

    It remains to bound the generalisation error of linear halfspace
    $\linear{k}$. Similar to the proof of~\Cref{thm:no-shift}, setting
    $\zeta = 1$ and $\eta = \frac{\alpha}{2}$ and
    invoking~\Cref{lem:generalization-bound} gives us that with
    probability $1-\frac{\beta}{4}$, \begin{equation}\label{eq:excess-err}\bP_{(x, y)\sim D_k}\bs{y\ba{v_k,
    x} < 0} - \frac{2}{\nlab} \sum_{(x, y)\in
    \Slab_k}\mathbbm{I}\{y\ba{v_k, x}< 0\} \leq \frac{\alpha}{2}.\end{equation}

    Thus, combining~\Cref{eq:shift-empirical-guarantee,eq:excess-err}
   we get \[\bP_{(x, y)\sim D}\bs{y\ba{v_k, \hA^\top x} < 0} 
 = \bP_{(x, y)\sim D_k}\bs{y\ba{v_k, x} < 0} \leq
   \frac{2}{\nlab} \sum_{(x, y)\in \Slab_k}\mathbbm{I}\{y\ba{v_k, x}<
   0\} + \frac{\alpha}{2} = \alpha,\] 
   for $\nlab \geq \frac{k}{\alpha}\text{polylog}\br{\frac{1}{\beta}, \frac{1}{k}}$. This is equivalent as stating that the output of~\Cref{alg:no_shift} $\hw = \hA v_k$ satisfies 
   \[\bP_{(x, y)\sim D}\bs{y\ba{\hw, x}<0}= \bP_{(x, y)\sim D}\bs{y\ba{\hA v_k, x}<0} =\bP_{(x, y)\sim D}\bs{y\ba{v_k, \hA^\top x}<0} \leq \alpha,\]
   which concludes the proof.
    \end{proof}

    \begin{proof}[Proof of~\Cref{lem:bounded-covariance-shift}]
        We first approximate Moment Generating Functions (MGFs) of $g$ and $f$ by their first and second moments. Then, we express the error bound in this approximation by the error bound for Taylor expansion, for any $t\in \mathbb{R}^d$ with $\norm{t}_2> 0$, 
        \begin{equation}
            \label{eq:taylor-expansion-f}
            \begin{aligned}
                \abs{M_f(t) -\ 1 + t^T \bE_f\bs{X} + \frac{t^T \cov_f t}{2} } &
                \overset{(a)}{\leq}  \frac{ \bE_f\bs{e^{t^Tx}xx^Tx}\norm{t}_2^3}{3!} \\
                &\overset{(b)}{\leq} \frac{\bE_f\bs{xx^Tx} e^{\norm{t}_2}\norm{t}_2^3}{3!}\\
                &\overset{(c)}{\leq}\eta \norm{t}_2^3 
            \end{aligned}
        \end{equation}
        where step $(a)$ follows by the error bound of Taylor expansion, step $(b)$ is due to $e^{t^Tx}\leq e^{\norm{t}_2\norm{x}_2}\leq e^{\norm{t}_2}$ for all $x\in B_d^2$, and step $(c)$ follows from $e^{\norm{t}_2}\leq 3! \text{ for }\norm{t}_2 \leq 1 $. 
        Similarly, 
        \begin{equation}
            \label{eq:taylor-expansion-g}
                \abs{M_g(t) -\ 1 + t^T \bE_g\bs{X} + \frac{t^T \cov_g t}{2} } 
                \leq \eta \norm{t}_2^3.
        \end{equation}
        Rewrite \cref{eq:taylor-expansion-f} and \cref{eq:taylor-expansion-g} and observing that $\bE_g[X] = \bE_f[X] = 0$, we can bound the terms $\frac{t^T \cov_f t}{2} $ and $\frac{t^T \cov_g t}{2} $ by 
        \begin{equation}
            \label{eq:mgf-f-g-bounds}
            \begin{aligned}
                1-M_f(t) - \eta\norm{t}_2^3 &\leq \frac{t^T \cov_f t}{2} \leq 1-M_f(t) + \eta\norm{t}_2^3 \\
                1-M_g(t) - \eta\norm{t}_2^3 &\leq \frac{t^T \cov_g t}{2} \leq 1-M_g(t) + \eta\norm{t}_2^3.
            \end{aligned}
        \end{equation}
        
      Next, we show that the discrepancy in covariance matrices of distributions $G$ and $F$ are upper bounded by the difference in their MGFs. By \cref{eq:mgf-f-g-bounds}, for all $t \in \bR^d$ and $\norm{t}_2 \neq 0$, 
        \begin{equation}
            \begin{aligned}
            \label{eq:bounded-second-moment-mgf}
                \abs{\frac{t^T\br{\cov_f - \cov_g}t}{2}}&\leq 1-M_f(t) + \eta\norm{t}_2^3 - 1+M_g(t)+ \eta\norm{t}_2^3 \\
                &= \abs{M_g(t) - M_f(t) + 2\eta \norm{t}_2^3}\\
                &\leq \abs{M_g(t) - M_f(t)} + 2\eta\norm{t}_2^3
            \end{aligned}
        \end{equation}
    
        Next, we upper bound the difference between the MGFs of distributions $G$ and $F$ by the TV distance between them. 
    
    \begin{equation}\label{eq:bound-diff-mgf}
        \begin{aligned}
            \abs{M_f(t) - M_g(t)} 
            &= \abs{\int_{x\in B^d_2}e^{t^T x} \bs{f(x)  -  g(x)} dx}\\
            &\leq \int_{x\in B^d_2}e^{t^T x} \abs{f(x)  -  g(x)} dx\\
            &\leq \int_{x\in B^d_2}e^{\norm{t}_2\norm{x}_2} \abs{f(x)  -  g(x)} dx\leq \frac{e^{\norm{t}_2}\eta}{2}
        \end{aligned}
    \end{equation}
    where the last inequality follows as $\norm{x}_2 = 1$ for $x\in B_2^d$ and $TV(f, g)\leq \eta$. 
    
    Combine \cref{eq:bounded-second-moment-mgf} and \cref{eq:bound-diff-mgf}, we have for all $t \in \bR^d$ and $\norm{t}_2\neq 0$, 
    \begin{equation}
        \label{eq:cov-bound-with-t}
        \abs{t^T\br{\cov_f - \cov_g}t} \leq e^{\norm{t}_2}\eta_1 +4\eta\norm{t}_2^3
    \end{equation}
    Choose $t$ as a vector in the direction of the first eigenvector (\ie~the eigenvector corresponding to the largest eigenvalue) of  $\cov_f - \cov_g$. For $t$ in this direction, by the definition of operator norm, 
    \begin{equation}\label{eq:exact_cauchy_schwarz}
        \norm{\cov_f - \cov_g}_{\op} = \frac{\abs{t^T\br{\cov_f - \cov_g}t} }{\norm{t}_2}.
    \end{equation}

    Plugging \cref{eq:exact_cauchy_schwarz} into \cref{eq:cov-bound-with-t} and choose the norm of $t$ as the minimizer of $ e^{\norm{t}_2}\eta_1 + 4\eta\norm{t}_2^3$, we get 
    
    \begin{equation*}
        \norm{\cov_f - \cov_g}_{\op} \leq \min_{0\leq \norm{t}_2\leq 1}\frac{e^{\norm{t}_2}\eta}{\norm{t}_2^2} + 4\eta\norm{t}_2 \leq \frac{\eta(1 + \norm{t}_2 + \norm{t}_2^2)}{\norm{t}_2^2} + 4\eta\norm{t}_2 = 7\eta
    \end{equation*}
    This conclude the proof. 
\end{proof}

\subsection{Large margin Gaussian mixture distributions}
\label{sec:appendix-gmm}
Here, we provide the sample complexity of \ouralgo~on the class of
large margin mixture of Gaussian mixture distributions defined
in~\Cref{defn:large-margin-gmm}. We present a corollary
following~\Cref{thm:no-shift}, which shows that for large margin
Gaussian mixture distributions, \ouralgo~leads to a drop in the
private sample complexity from $O(\sqrt{d})$ to $O(1)$. 

    \begin{corollary}[Theoretical guarantees for large margin Gaussian mixture distribution]
        \label{corollary:gmm-full}
        For $\theta, \sigma^2 = \tilde{O}\br{\nicefrac{1}{\sqrt{d}}}$,
        let $\cD_{\theta, \sigma^2}$ be the family of all $(\theta,
        \sigma^2)$-large margin Gaussian mixture distribution
        (\Cref{defn:large-margin-gmm}). For any $\alpha \in (0, 1),
        \beta \in (0, \nicefrac{1}{4}), \epsilon \in (0,
        \nicefrac{1}{\sqrt{M}})$,
        and $\delta \in (0, 1)$,~\Cref{alg:no_shift} $\cA_{\epsilon, \delta}(k, \gamma(1-0.1\gamma))$ is
        an $(\alpha, \beta, \epsilon, \delta)$-semi-private learner on
        $D_{\theta, \sigma^2}$ of linear halfspaces $\linear{d}$ with
        sample complexity
        \begin{equation}
            \label{eq:example-sample-complexity-full}
            \begin{aligned}
            \nunl &= \bigO{\frac{M^2\log\frac{2}{\beta}}{\gamma^2\theta^2}},\\ \nlab &= \tilde{O}\br{\frac{M\sqrt{k}}{\alpha\epsilon \gamma(1-0.1\gamma)}}
            \end{aligned}
        \end{equation}
        where $\gamma = 1 - \br{4\sqrt{d} + 2\sqrt{\log\frac{2\nlab}{\delta}}}\br{\sigma^2 + \theta}$, $M = 1 +\br{4\sqrt{d} + 2\sqrt{\log\frac{2\nlab}{\delta}}}\br{\sigma^2 + \theta}$. 
    \end{corollary}
    Here, in line with the notation
of~\Cref{defn:low-dim-large-margin}, $\gamma$ intuitively represents
the margin in the $d$-dimensional space and $M$ is the upper bound for
the radius of the labelled dataset. For $\theta = \sigma^2 =
\nicefrac{1}{2C\sqrt{d}}$ and ignoring the logarithmic terms, we get
$M = 1.5$ and $\gamma = 0.5$. \Cref{corollary:gmm-full} implies the
labelled sample complexity
$\tilde{O}\br{\nicefrac{1}{\alpha\epsilon}}$. 

    \begin{proof}
        To prove this result, we first show that all large-margin
        Gaussian mixture distributions $D_{\theta, \sigma^2}\in
        \cD_{\theta, \sigma^2}$ are $(\gamma_0, \xi)$-large margin low
        rank distribution~(\Cref{defn:low-dim-large-margin}) after
        normalization. In particular, we show that the normalized
        distribution is $(\gamma_0, \xi)$-large margin low rank
        distribution with $\xi = 0$ and margin $\gamma_0 =
        \nicefrac{\gamma}{M}$, where $\gamma = 1 - \br{4\sqrt{d} +
        2\sqrt{\log\frac{2\nlab}{\delta}}}\br{\sigma^2 + \theta}$ and
        $M = 1 +\br{4\sqrt{d} +
        2\sqrt{\log\frac{2\nlab}{\delta}}}\br{\sigma^2 + \theta}$.
        Then, invoking~\Cref{thm:no-shift} gives the desired sample
        complexity in~\cref{eq:example-sample-complexity-full}. 

        To normalize the distribution, we consider the marginal
        distribution \(D_X\) of the mixture distribution \(D \in
        \cD_{\theta, \sigma^2}\) and compute its mean and the
        covariance matrix. By~\Cref{defn:large-margin-gmm},  \(D\) is
        a mixture of two gaussians with identical covariance matrix
        $\Sigma = \theta \tw(\tw)^{\top} - \sigma^2 I_d$ and means
        $\mu_1 = -\mu_2$.With a slight misuse of notation, we denote
        the probability density function of a normal distribution with
        mean \(\mu\) and covariance \(\Sigma\) using
        \(\cN(x;\mu,\Sigma)\). Then, we can calculate the mean and
        covariance matrix as

        \begin{equation}
            \label{eq:marginal_mean}
            \bE_{X}\bs{X} = \bE_y\bE_{X|y}\bs{X|y} = \frac{1}{2} \mu_1 + \frac{1}{2}\mu_2 = 0
        \end{equation}
        and
        \begin{equation*}
            \label{eq:marginal_cov}
            \begin{aligned}
                \cov_X 
                &= \bE_{X}\bs{XX^{\top}} - \br{\bE_{X}\bs{X}}\br{\bE_{X}\bs{X}}^{\top}
                \overset{(a)}{=} \bE_{y}\bE_{X|y}\bs{XX^{\top}|y}\\
                &= \frac{1}{2}\int_{B_2^d} xx^{\top} \cN(x; \mu_1, \cov)dx+ \frac{1}{2}\int_{B_2^d} xx^{\top} \cN(x; \mu_2, \cov)dx \\
                &\overset{(b)}{=} \frac{1}{2}\br{\cov + \mu_1\mu_1^{\top}} + \frac{1}{2}\br{\cov + \mu_2\mu_2^{\top}}\\
                &\overset{(c)}{=} \theta \tw\br{\tw}^{\top} + \mu_1 \mu_1^{\top} +  \sigma^2I_d 
            \end{aligned}
        \end{equation*}
        where step $(a)$ follows by~\Cref{eq:marginal_mean}, step $(b)$ follows by the
        relationship between covaraince matrix and the second moment
        $\cov = \bE_X\bs{XX^{\top}} - \mu\mu^{\top}$, and step $(c)$ follows by the definition of large-margin Gaussian mixture distribution (\Cref{defn:large-margin-gmm}) of $\cov$ and $\mu_1, \mu_2$.

        Then, we show that the first two eigenvectors are $ \mu_1$ and
        $\tw$ with the corresponding eigenvalues $1 + \sigma^2$ and
        $\theta + \sigma^2$ for $\theta =
        \bigO{\nicefrac{1}{\sqrt{d}}}\leq 1$. The remaining non-spiked
        eigenvalues are $\sigma^2$. 
        \begin{align*}
            \cov_X \mu_1 &= \theta \tw\br{\tw}^{\top} \mu_1 + \mu_1 \mu_1^{\top} \mu_1 + \sigma^2 \mu_1\\
            &\overset{(a)}{=} (\norm{\mu_1}_2^2 + \sigma^2) \mu_1 = (1+\sigma^2)\mu_1\\ 
            \cov_X \tw &= \theta \tw\br{\tw}^{\top} \tw + \mu_1 \mu_1^{\top} \tw +\sigma^2 \tw\\
            &\overset{(b)}{=} (\theta + \sigma^2)\tw,
        \end{align*}
        where step $(a)$ and $(b)$ both follow from the fact that $\br{\tw}^{\top}\mu_1 = 0$. For $k = 2$, it follows immediately that $\deltak = \theta$ (\cref{eq:deltak-bound}) and $\xi = 0$ (\cref{eq:xi-bound}),
        \begin{equation}
            \label{eq:deltak-bound}
            \deltak = \eigval{k}{\cov_X}-\eigval{k+1}{\cov_X} = \theta + \sigma^2 - \sigma^2 = \theta.
        \end{equation}
        \begin{equation}
            \label{eq:xi-bound}
            \begin{aligned}
                \frac{\norm{\A^{\top} \tw}_2}{\norm{\tw}_2} &= \frac{1}{\norm{\tw}_2}\begin{bmatrix}
                    \mu_1^{\top} \\
                    \br{\tw}^{\top}
                \end{bmatrix}\tw\\
                &=\frac{\abs{\mu_1^{\top} \tw + \br{\tw}^{\top}\tw}}{\norm{\tw}_2} \\
                &\overset{(a)}{=}1 = 1-\xi, 
            \end{aligned}
        \end{equation}
        where step $(a)$ follows from $\mu_1^{\top}w = 0$. 
        
        Next, we show that the labelled dataset lies in a ball with bounded radius with high probability, which further implies that original data has a large margin. 

        Denote the part of the dataset from the gaussian component
        with $y = 1$ by $\Slab_1$ and denote the part from the
        component with $y = -1$ by $\Slab_2$.  We apply the well-known
        concentration bound on the norm of Gaussian random vectors
        (\Cref{lem:gaussian_norm_bounds}) to show a high probability
        upper bound on the radius of the datasets $\Slab_1$ and
        $\Slab_2$. 
        \begin{lem}[\citep{vershynin_2018}]\label{lem:gaussian_norm_bounds}
            Let $X\sim N(\mu, \Sigma)$, where $v \in B^2_d$. Then, with probability at least $1-\delta$, 
            \begin{equation*}
                \label{eq:gaussian_norm_concentration}
                \norm{X - \mu}_2\leq 4\norm{\Sigma}_{\op}\sqrt{d} + 2\norm{\Sigma}_{\op}\sqrt{\log \frac{1}{\delta}}.
            \end{equation*}
        \end{lem}
        This gives the following high probability upper bound on any $x\in \Slab_i$ for $i = {1, 2}$ and some $\frac{\beta}{2\nlab} > 0$, 
        \[\bP_{\Slab\sim D^{\nlab}}\bs{\norm{x-\mu_i}_2\leq
        4\br{\theta + \sigma^2}\sqrt{d} + 2\br{\theta +
        \sigma^2}\sqrt{\log\frac{4\nlab}{\beta}}}\geq
        1-\frac{\beta}{4\nlab}\] For $i \in \{1, 2\}$, by applying
        union bound on all $x\in \Slab_i$, we can bound maximum
        distance of a points $x\in \Slab_i$ to the center $\mu_i$, 
        \begin{equation*}
            \label{eq:comp-1-norm-bound}
            \bP_{\Slab\sim D^{\nlab}}\bs{\max_{x\in \Slab_i}\norm{x-\mu_i}_2\leq \br{\theta + \sigma^2}\br{4\sqrt{d} + 2\sqrt{\log\frac{4\nlab}{\beta}}}}\leq 1-\frac{\beta}{4}. 
        \end{equation*}
        Note that the distance between the two centers $\mu_1$ and $\mu_2$ is $2$. Thus, with probability at least $1-\frac{\beta}{2}$, all points in the labelled dataset $\Slab$ lie in a ball centered at $0$ having radius 
        \begin{equation*}\label{eq:gmm-radius}
            M = 1 +\br{4\sqrt{d} + 2\sqrt{\log\frac{4\nlab}{\beta}}}\br{\sigma^2 + \theta}.
        \end{equation*}
            
        Also, the margin in the original labelled dataset is at least 
        \begin{equation*}
            \label{eq:gmm-margin-original}
            \gamma = 1 -\br{4\sqrt{d} + 2\sqrt{\log\frac{4\nlab}{\beta}}}\br{\sigma^2 + \theta}. 
        \end{equation*}

        Normalizing the data by $M$, it is obvious that the normalized distribution satisfies the definition of $(\gamma, \xi)$-large margin low rank distribution with parameters $\xi = 0$, $\deltak = \nicefrac{\theta}{M}$ and $\gamma_0 = \nicefrac{\gamma}{M}$, where $\gamma = 1 - \br{4\sqrt{d} + 2\sqrt{\log\frac{2\nlab}{\delta}}}\br{\sigma^2 + \theta}$, $M = 1 +\br{4\sqrt{d} + 2\sqrt{\log\frac{2\nlab}{\delta}}}\br{\sigma^2 + \theta}$. 
    Invoking~\Cref{thm:no-shift} concludes the proof. 
\end{proof}

\subsection{Discussion of assumptions for existing methods}\label{sec:appendix-comparison}
\begin{table}[t]\small
    \begin{center}
        \begin{tabular}{ C{3.8cm} C{1.6cm} C{4.5cm} c}
        \toprule
          & Public Unlabelled Data & Low-rank Assumption & Sample complexity\\ 
          \midrule
         Generic semi-private learner~(\citet{alon2019limits})&\CheckmarkBold&- &$\tilde{O}\br{\dfrac{\sqrt{d}}{\alpha\epsilon\gamma}}$ \\
         \midrule \multicolumn{4}{c}{\normalsize \bfseries No Projection}\\
         DP-SGD~\citet{li2022when} & \XSolidBrush & Restricted
         Lipschitz Continuity
         $\br{\sum_{i = 1}^{\log(d/k)+1}G_{2^{s-1}k}^2 \leq c_2 }$ &
         $\tilde{O}\br{\dfrac{\sqrt{k}}{\alpha\epsilon\gamma} +
         \sqrt{\dfrac{c_2 d}{k}}}$\\
         \midrule
         \multicolumn{4}{c}{\normalsize\bfseries Random Projection~(e.g. JL Transform)}\\
         \citet{nguyen19jl} & \XSolidBrush & - & $\tilde{O}\br{\dfrac{1}{\alpha\epsilon \gamma^2}}$\\  
        \citet{kasiviswanathan21a} & \XSolidBrush  &- & $\tilde{O}\br{\min\left\{\dfrac{\omega(\cC)}{\beta}, \sqrt{d}\right\}\dfrac{1}{\alpha\epsilon\gamma}}$\\
         \midrule\multicolumn{4}{c}{\normalsize\bfseries Low Rank Projection Projection}\\
         GEP~\citet{yu2021do}  & \CheckmarkBold  &Low-rank gradients $\br{\bar{r}\leq c_1}$ & $\tilde{O}\br{\dfrac{1}{\alpha\epsilon\gamma} +(\sqrt{k} + c_1\sqrt{d})}$ \\
         \textbf{OURS} & \CheckmarkBold & Low Rank Separability~(\Cref{defn:low-dim-large-margin}) & $\tilde{O}\br{\dfrac{\sqrt{k}}{\alpha\epsilon\gamma(1-\xi-0.1\gamma)}}$\\
        \bottomrule
        \end{tabular}
        \end{center}
        \caption{Comparison with existing works: $\omega\br{\cC}$
        represents the Gaussian width of the parameter space $\cC$,
        and $c_1, c_2$ are constants that decrease with the
        low-rankness of the gradient space of the loss
        function.~\(G_i\) represents the projection of the norm of the
        projection of the gradient onto the null space of a low rank
        matrix and is formally defined
        in~\Cref{eq:RCL-coef-definition}. All remaining notations:
        \(d,k,\xi,\gamma,\alpha,\beta,\) and \(\epsilon\) have the
        same meaning as the main text. }
        \label{table:comparison_theory}
\end{table}

\Cref{table:comparison_theory} summarises the comparison of our
theoretical results with some existing methods. We describe the
notation used in the table below.

\paragraph{Analysis of the Restricted Lipschitz Continuity (RLC) 
assumption~\citep{li2022when}}

As indicated in~\cref{table:comparison_theory},
DP-SGD~\citep{li2022when} achieves dimension independent sample
complexity if the following assumption, known as Restricted Lipschitz
Continuity~(RLC) is satisfied. For some $k\ll d$, 
\begin{equation}
    \label{eq:RLC-assumption}\tag{RLC~1}
    \sum_{i = 1}^{\lfloor\log(d/k)+1\rfloor}G_{2^{i-1}k}^2 \leq O(\sqrt{\nicefrac{k}{d}}),
\end{equation}
where $G_0, G_1, ..., G_d$ represent the RLC coefficients. For any
$i\in [d]$, the loss function \(\ell\) is said to satisfy RLC with
coefficient $G_i$ if
\begin{equation}
    \label{eq:RCL-coef-definition}
     \quad G_i \geq \min_{\substack{\rank{P_i}=i\\P_i\in\Pi}} \norm{(I - P_i)\nabla \ell(w; (x, y))}_2, 
\end{equation}
for all \(w,x,y\in\text{domain}\br{\ell}\), where  \(\Pi\) is the set of orthogonal projection matrices.
Equivalently, assumption~\ref{eq:RLC-assumption} states that for some
$k\ll d$, 
\begin{equation}
    \label{eq:RLC-assumption-simplified}\tag{RLC}
    \sum_{i = k+1}^{d}G_{i}^2 \leq O(\sqrt{\nicefrac{k}{d}}).
\end{equation}

In this section, we demonstrate that if we assume the Restricted
Lipschitz Continuity (RLC) condition from~\citet{li2022when}, our low
rank separability assumption on $\norm{A_k A_k^\top \tw}$ holds for
large-margin linear halfspaces. However, using the RLC assumption
leads to a looser bound compared to our assumption. More
specifically, given the~\ref{eq:RLC-assumption-simplified}
assumption and the loss function $\ell$ defined in
\Cref{eq:loss-funtion-alg}, we can show $\norm{A_k A_k^\top \tw}\geq
\gamma$.

\noindent Given the parameter $\zeta$ in~\cref{alg:no_shift}, for $x,
y\in \supp{D}$ and $w$ satisfying $y\ba{w, x}\leq \zeta$, we can
calculate the $i^{\text{th}}$ restricted Lipschitz coefficient
\begin{equation}
    \label{eq:loss-function-rlc-bound}
    \begin{aligned}
        G_i &\geq \min_{\substack{\rank{P_i}=1\\P_i\in\Pi}} \norm{(I - P_i)\nabla \ell(w; (x, y))}_2 \\
        &= \min_{\substack{\rank{P_i}=1\\P_i\in\Pi}}\norm{\frac{y}{\zeta}(I - P_i)x}_2\\
        &= \min_{\substack{\rank{P_i}=1\\P_i\in\Pi}} \norm{\frac{1}{\zeta}\br{x - P_ix}}_2.
    \end{aligned}
\end{equation}

\noindent Equivalently, we can rewrite
\cref{eq:loss-function-rlc-bound} as there exists a rank-\(i\) orthogonal
projection matrix \(\Pigood\) such that
\begin{equation}
    \label{eq:gradient-rank-i-bound}
    \norm{x - \Pigood x}_2 \leq \zeta G_i.
\end{equation}

\noindent Thus, for $x$ such that $y\ba{w, x}\leq \zeta$, 
\begin{equation}\label{eq:rlc-upper-bound1}
    \begin{aligned}
        \norm{xx^\top - (\Pigood x)(\Pigood x)^\top }_\op &\overset{(a)}{=} \norm{(x-\Pigood x)(x+\Pigood x)^\top}_2 \\
        &\leq \norm{x+\Pigood x}_2\norm{x - \Pigood x}_2\\
        &\overset{(b)}{\leq} 2\norm{x - \Pigood x}_2 \\
        &\overset{(c)}{\leq} 2G_i \zeta 
    \end{aligned}
\end{equation}
where step $(a)$ follows from the orthogonality of $\Pigood$, step $(b)$ follows from $\norm{\Pigood x}_2\leq \norm{x}_2 = 1$, and step $(c)$ follows from~\cref{eq:gradient-rank-i-bound}. 

Then, we can bound the low-rank approximation error for the covariance matrix of the data distribution. 
\begin{equation*}
    \begin{aligned}
        \norm{\cov_X - \Pigood \cov_X (\Pigood) ^\top}_{\op} \overset{(a)}{\leq} \bE_{x\sim D_X}\br{\norm{xx^\top - (\Pigood x)(\Pigood  x)^\top }_{\op}} \overset{(b)}{\leq} 2G_i \zeta.
    \end{aligned}
\end{equation*}
where $\cov_X = \bE_{x\sim D_X}\bs{xx^\top}$, and step $(a)$ follows from the convexity of the Euclidean norm and step $(b)$ follows from~\Cref{eq:rlc-upper-bound1}. 

This further provides an upper bound on the last $d-k$ eigenvalues of the covariance matrix $\cov_X$ of the data distribution $D_X$. Let $\lambda_i$ denote the $i^{\text{th}}$ eigenvalue of the covariance matrix $\cov_X$. Then, we apply~\Cref{lem:rank-k-approximation} that gives an upper bound on the singular values of a matrix in terms of the rank $k$ approximation error of the matrix.
\begin{lem}[\cite{lec521}]\label{lem:rank-k-approximation}
    For any matrix $M\in \bR^{m\times n}$, 
    \begin{equation*}
        \inf_{\rank{\hat{M}} = k}\norm{M - \hat{M}}_{\op} = \sigma_{k+1},
    \end{equation*}
    where the infimum is over all rank $k$ matrices $\hat{M}$ and $\sigma_{k+1}$ is the $k^{\text{th}}$ singular value of the matrix $M$. 
\end{lem}

This gives an upper bound on the $i^{\text{th}}$ eigenvalue of the covariance matrix $\cov_X$ in terms of the $i^{\text{th}}$ restricted Lipschitz coefficient, 
\begin{equation*}
    \lambda_{i+1} = \sigma_{i+1}^2 = \inf_{\rank{\cov_X'} = i} \norm{\cov_X - \cov_X'}_{\op}^2\leq \norm{\cov_X - \Pigood \cov_X (\Pigood)^\top}^2_{\op}\leq 4G_i^2\zeta^2.
\end{equation*}

Thus, for matrix $A_k$ consisting of the first $k$ eigenvectors of $\cov_X$, we can upper bound the reconstruction error of $A_k^\top x$ with the eigenvalues of the covariance matrix $\cov_X$,  

\begin{equation*}
    \begin{aligned}
        \bE_{x\sim D_X}\bs{\norm{x}_2 - \norm{A_k^\top x}_2} &= \bE_{x\sim D_X}\bs{\norm{xx^\top}_{\op} - \norm{(A_k x)(A_k x)^\top}_{\op}} \\
        &\leq \bE_{x\sim D_X}\bs{\norm{xx^\top - (A_k x)(A_k x)^\top}}_{\op}\leq \sum_{i = k+1}^d \lambda_i \leq 4\zeta^2\sum_{i = k+1}^d G_i^2. 
    \end{aligned}
\end{equation*}
By Markov's inequality, with probability at least $1-\beta$, 
\begin{equation}\label{eq:hp-low-rank-property-uppper}
    \begin{aligned}
        &\bP_{x\sim D_X}\bs{\norm{xx^\top}_{\op} -\norm{(A_k^\top x)(A_k^\top x)^\top}_{\op}\geq \frac{4\zeta^2}{\beta}\sum_{i = k+1}^dG_i^2} \\
        &\leq \bP_{x\sim D_X}\bs{\norm{xx^\top - (A_k^\top x)(A_k^\top x)^\top}_{\op} \leq\frac{4\zeta^2}{\beta}\sum_{i = k+1}^dG_i^2} \leq \beta. 
    \end{aligned}
\end{equation}

This implies our assumption with probability at least $1-\beta$, 
\begin{equation}\label{eq:rlc-our-assumption}
    \begin{aligned}
        \norm{A_k A_k^\top \tw}_2 &\overset{(a)}{=} \norm{x}_2\norm{A_k A_k^\top \tw}_2 
        \geq |\ba{A_k A_k^\top x, \tw}|\\
        &\overset{(b)}{\geq} |\ba{x, \tw}| - |\ba{x - A_k A_k^\top x, \tw}|\\
        &\overset{(c)}{\geq} \gamma - \norm{x - A_k A_k^\top x}_2\norm{\tw}_2 \\
        &\overset{(d)}{\geq} \gamma - \frac{4
        \zeta^2}{\beta}\sum_{i = k+1}^dG_i^2
    \end{aligned}
\end{equation}
where step $(a)$ follows from $ \norm{x}_2 = 1$, step $(b)$ follows by $\ba{A_k A_k^\top x, \tw} = \ba{x, \tw}-\ba{x - A_k A_k^\top x, \tw}$ and the triangle inequality, step $(c)$ follows by the large margin assumption $y\ba{x, \tw} = \abs{\ba{x, \tw}}\geq \gamma$, and step $(d)$ follows by~\cref{eq:hp-low-rank-property-uppper} with probability at least $1-\beta$. 

The RLC assumption requires the last term
in~\cref{eq:rlc-our-assumption} to vanish at the rate of
$O(\nicefrac{k}{d})$. This implies our low-rank assumption holds with
$\xi = 1-\gamma$. 

\paragraph{Analysis on the error bound for GEP}To achieve a dimension-independent sample complexity bound in GEP~\cite{yu2021do}, the gradient space must satisfy a low-rank assumption, which is even stronger than the rapid decay assumption in RLC coefficients (\cref{eq:RLC-assumption}). By following a similar argument as the analysis for the RLC assumption~\cite{li2022when}, we can demonstrate that our low-rank assumption is implied by the assumption in GEP.

\clearpage
\section{Experimental details and additional experiments}
\subsection{Details and hyperparameter ranges for our method}
\label{app:hyperparam-cifar10-100}

Unless stated otherwise, we use the PRV
accountant~\citep{PRVAccountant} in our experiments. Following
\citet{de2022unlocking}, we use the validation data for
cross-validation of the hyperparameters in all of our experiments and
set the clipping constant to $1$. We search the learning rate in
$\bc{0.01,0.1,1}$, use no weight decay nor momentum as we have seen it
to have little or adverse impact. We search the number of steps in
$\bc{500, 1000,3000,5000, 6000}$ and our batch size in
$\bc{128,512,1024}$. We compute the variance of the noise as a
function of the number of steps and the target $\epsilon$ using
\texttt{opacus}. We set $\delta=1e-5$ in all our experiments.  We use
the open-source \texttt{opacus} \citep{opacus} library to run DP-SGD
with the PRV Accountant efficiently.  We use \texttt{sciki-learn} to
implement PCA. Checkpoints of ResNet-50 are taken or trained using the
\texttt{timm}~\citep{rw2019timm} and
\texttt{solo-learn}~\citep{solo-learn} libraries. Standard ImageNet
pre-processing of images is applied, without augmentations.

\subsection{Discrepancy in pre-training resolution}
\label{app:details-cifar10-100}

Several works have used different resolutions of ImageNet to pre-train
their models. In particular,~\citet{de2022unlocking} used ImageNet
32x32 to pre-train their model, which is a non-standard dimensionality
of ImageNet, but it matches the dimensionality of their private
dataset CIFAR-10. In contrast, we use the standard ImageNet (224x224)
for pre-training in all our experiments with both CIFAR datasets as
well as other datasets. In this section, we show that using the
resolution of 32x32 for pre-training, we can indeed
outperform~\citet{de2022unlocking} but also highlight why this may not
be suitable for privacy applications.

\begin{table}
    \centering
    \begin{tabular}{c|cc|cc} 
    \toprule
    Privacy& \multicolumn{2}{c}{CIFAR10} & \multicolumn{2}{c}{CIFAR100} \\
     &~\textbf{Ours}
    &~\textbf{\citet{de2022unlocking}}&~\textbf{Ours}
    &~\textbf{\citet{de2022unlocking}}\\ 
    \hline
    \hspace{15pt}\(\epsilon=0.1\)\hspace{15pt} & 89.4 & -&36.1&-\\
    \(\epsilon=0.7\)& 93.1 & -&69.7&-\\
    \(\epsilon=1\) & 93.5 &93.1&71.8&70.3\\
    \(\epsilon=2\) & 93.9 & 93.6&74.9&73.9\\
    \bottomrule
    \end{tabular}
    \caption{Result for our algorithm is with pre-training on
    ImageNet32x32. Results for~\citet{de2022unlocking} is taken from
    their paper where available.}
    \label{tab:app-baselines}
    \end{table}

\paragraph{Low-resolution~(CIFAR specific) pre-training} Different
private tasks/datasets may have images of differing resolutions. While
all images in CIFAR~\citep{cifar10} are 32x32 dimensional, in other
datasets, images not only have higher resolution but their resolution
varies widely. For example, GTSRB~\citep{GTSRB} has images of size
222x193 as well as 15x15 , PCAM~\citep{PCAM} has 96x96 dimensional
images, most images in Dermnet~\citep{DermNet} have resolution larger
than 720x400, and in Pneumonia~\citep{kermany2018identifying} most x-rays have a
dimension higher than 2000x2000. Therefore, it may not be possible to
fine-tune the feature extractor at a single resolution for such
datasets. 

Identifying the optimal pre-training resolution for each private
dataset is beyond the scope of our work and orthogonal to the
contributions of our work (as we extensively show, our
method~\ouralgo~ operates well under several pre-training strategies
in~\Cref{fig:neardist_app} and \Cref{fig:farshift_app}). Furthermore,
assuming the pre-training and private dataset resolution to be
perfectly aligned is a strong assumption.

\paragraph{Comparison with~\citet{de2022unlocking}} Nevertheless, we
compare our approach with~\citet{de2022unlocking} pre-training a
ResNet50 on a 32x32 rescaled ImageNet version, and obtain a
non-private accuracy larger than \(94\%\) reported for \(\epsilon=8\)
in Table 5 in~\citet{de2022unlocking} for~\emph{Classifier training}.
Note that our approaches is computationally significantly cheaper than
theirs as we do no use the tricks proposed in their work (including
Augmult, EMA, and extremely large batch sizes ($>$ 16K))

Using ImageNet32x32 for pre-training, we perform slightly better than
them in private training. Our results are reported
in~\Cref{tab:app-baselines}. We expect that applying their techniques
will result in even higher accuracies at the cost of computational
efficiency.  Interestingly,~\Cref{tab:app-baselines} shows that our
model's accuracy for~\(\epsilon=0.7\) on CIFAR10, is as good
as~\citet{de2022unlocking} for \(\epsilon=1.0\). This provides
evidence that large batch sizes, which is one of the main hurdles in
producing deployable private machine learning models, might not be
required using our approach.

\subsection{Experiments with large \(\epsilon~(\geq 1)\)}
\label{app:details-cifar10-100-large-eps}

While in most of the paper, we focus on settings with small
\(\epsilon\), in certain practical settings, the large epsilon regime
may also be important. In~\Cref{tab:app-baselines-epsilon}, we repeat
our experiments for CIFAR10 and CIFAR100 with $\epsilon\in{1,2}$ and
report the accuracy for the best projection dimension. Our results
show that for \(\epsilon\in\bc{0.1,0.7,1,2}\) our method can provide
significant gains on the challenging dataset of CIFAR-100; however for
CIFAR-10 with $\epsilon=1,2$ the improvements are more modest.

\subsection{Comparison with PATE}
\label{sec:pate}
We now discuss discuss the \emph{PATE} family of
approaches~\citep{papernot2017semisupervised,papernot2018scalable,zhu2020private,muhl2022personalized}.
These methods partition the training set into disjoint subsets, train
an ensemble of teacher models on them, and use them to pseudo-label a
public dataset using a privacy-preserving mechanism. For PATE to
provide tight privacy guarantees, a large
number~(150-200~\citep{papernot2018scalable}) of subsets is needed,
which reduces the test accuracy of each teacher. Large amounts of
public data are also required. For
CIFAR-10,~\citet{papernot2018scalable,zhu2020private} use 29000
examples (58\% of training set size), whereas we only use 5000 (10\%
of training set size) public unlabelled data points (and to retain its
accuracy, in \Cref{sec:exp-low-data} we show  500 (1\%) samples are
sufficient). Of these 29000 examples,~\citet{zhu2020private} reports
only half of them is labelled due to the  private labelling mechanism,
further limiting the student's performance in settings with low
amounts of public training data. Despite our best attempts, we could
not train PATE-based approaches in our challenging setting to
satisfactory levels of accuracy on either CIFAR-10 or
CIFAR-100.\footnote{For reference, we refer the reader to the
accuracies reported for the state-of-the-art implementation
in~\citet{boenisch2023have} (Table 12) and~\citet{zhu2020private}
(Table 1), which are less than 40\% and 75\% respectively, whereas we
obtain more than 85\% for tighter privacy guarantees.
}

\begin{table}
    \centering
    \begin{tabular}{cc|cc|cc} 
        \toprule
        Privacy& Pre-training &\multicolumn{2}{c}{CIFAR10} & \multicolumn{2}{c}{CIFAR100} \\
        &&~\textbf{Ours}&No Projection&~\textbf{Ours}&No Projection\\ 
        \hline
        \multirow{2}{*}{\(\epsilon=1\)}&SL &86.4&85.4& 58.8 & 54.4\\
        &SSL &81.4&80.5&49.0 & 45.8\\\midrule
        \multirow{2}{*}{\(\epsilon=2\)}&SL &86.8 &86.4&61.8 & 60.0\\
        &SSL&82.5 &81.9&53.03 & 50.06\\
        \bottomrule
        \end{tabular}
        \caption{Experiment with larger \(\epsilon\). Pre-training is with ImageNet 224x224.}
        \label{tab:app-baselines-epsilon}
\end{table}

\subsection{Experimental details for~\Cref{sec:comp-exp}}
\label{app:details-comp-exp}

In this section, we provide details of the other algorithms we compare
our approach with in~\Cref{sec:comp-exp}. To ensure a fair comparison
for all the baseline methods, which independently implement most of
the private training algorithms (including per-sample gradient
computation, clipping, and noise addition), we limit our comparison to
the setting where $\epsilon=0.7$. We use the RDP
accountant~\citep{mironov2017renyi} for this specific comparison,
while in the rest of the paper, we employ the more recent PRV
accountant~\citep{gopi2021numerical}. The reason behind this choice is
that the code of the other methods is available with the RDP
accountant, and the implementation of the RDP accountant is unstable
for small~\(\epsilon\) values~\(\epsilon~(=0.1)\).

\paragraph{JL transformation~\citep{nguyen19jl}}~\citet{nguyen19jl}
uses JL transformation to reduce the dimensionality of the input.  For
our baseline, we simulate this method by using Random Matrix
Projection using Gaussian Random Matrices (GRM) instead of PCA to
reduce the dimensionality of the inputs. Our experimental results
in~\Cref{tab:baselines} show that our method outperforms these
approaches. Although this approach does not require the availability
of public data, this comparison allows us to conclude that reducing
the dimensionality of the input is not sufficient to achieve improved
performance.  Furthermore, even though the JL Lemma~\citet{JLLemma}
guarantees distances between inputs are preserved up to a certain
distortion in the lower-dimensionality space, the dataset size
required to guarantee a small distortion is much larger than what is
available in practice. We leverage \texttt{scikit-learn} to project
the data to a target dimension identical to the ones we use for PCA.
We similarly search the same hyperparameter space.

\paragraph{GEP~\citep{yu2021do}} We use the
code-base\footnote{\url{https://github.com/dayu11/Gradient-Embedding-Perturbation}}
released by the authors for implementation of GEP. We conduct
hyper-parameter search for the learning rate in $\{0.01,0.05,0.1,1\}$
and the number of steps in $\{500,1000,2500,3000,5000,6000,20000\}$.
As recommended by the authors, we set the highest clipping rate to
$\{1,0.1,0.01\}$ and the lowest clipping rate is obtained by
multiplying the highest with $0.20$.  The anchor dimension ranges in
$\{40,120,200,280,512,1024,1580\}$. We try batch sizes in
$\{64,512,1024\}$. We tried using $\{0.1\%,0.01\%\}$ of the data as
public. Despite this extensive hyperparameter search, we could not
manage to make GEP achieve better performance than the DP-SGD
baseline~(see~\Cref{tab:baselines}). 

\paragraph{AdaDPS~\citep{li2022private}} We use the
code-base\footnote{\url{https://github.com/litian96/AdaDPS}} released
by the authors of AdaDPS. We estimate the noise variance as a function
of the number of steps and the target $\epsilon$ using the code of
\texttt{opacus} under the RDP accountant (whose implementation is the
same as the code released by the authors of AdaDPS). We search the
learning rate in $\{0.01,0.1,1\}$, the number of steps in
$\{500,1000,2500,3000,5000,6000, 7500,10000\}$, the batch size in
$\{32,64,128,512,1024\}$, and we tried using $\{0.1\%,0.01\%\}$ of the
data as public, and the $\epsilon_c$ (the conditioner hyperparameter)
in $\{10,1,0.1,1e-3,1e-5,1e-7\}$. of the validation data for the
public data conditioning. We applied micro-batching in $\{2,4,32\}$.
Despite this extensive hyperparameter search, we could not manage to
make AdaDPS achieve better performance than the DP-SGD baseline.

\subsection{Different pre-training algorithms}
\label{sec:add_exp}

\begin{figure}[t]
    \centering
    \begin{subfigure}[b]{.48\linewidth}
    \includegraphics[width=\linewidth]{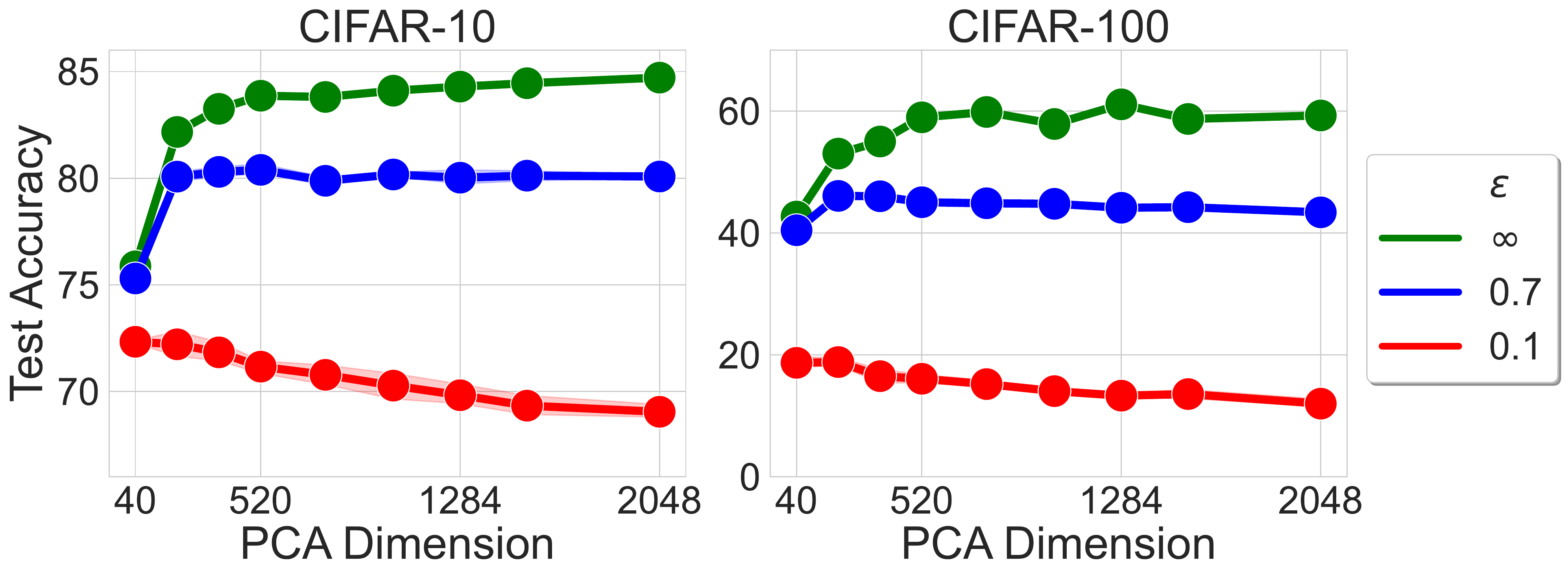}
    \subcaption{BYOL Pretraining}
    \end{subfigure}\hfill
    \begin{subfigure}[b]{.48\linewidth}
    \includegraphics[width=\linewidth]{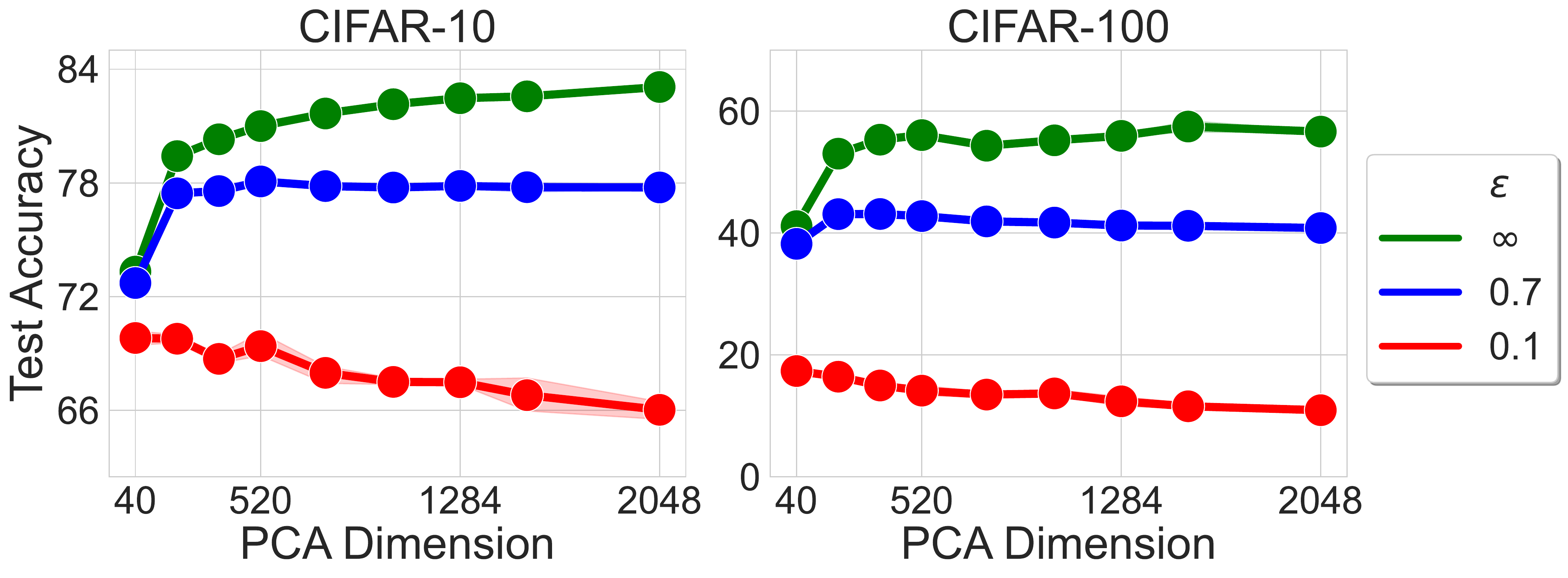}
    \subcaption{MoCov2+ Pretraining}
    \end{subfigure}
    \begin{subfigure}[b]{.48\linewidth}
    \includegraphics[width=\linewidth]{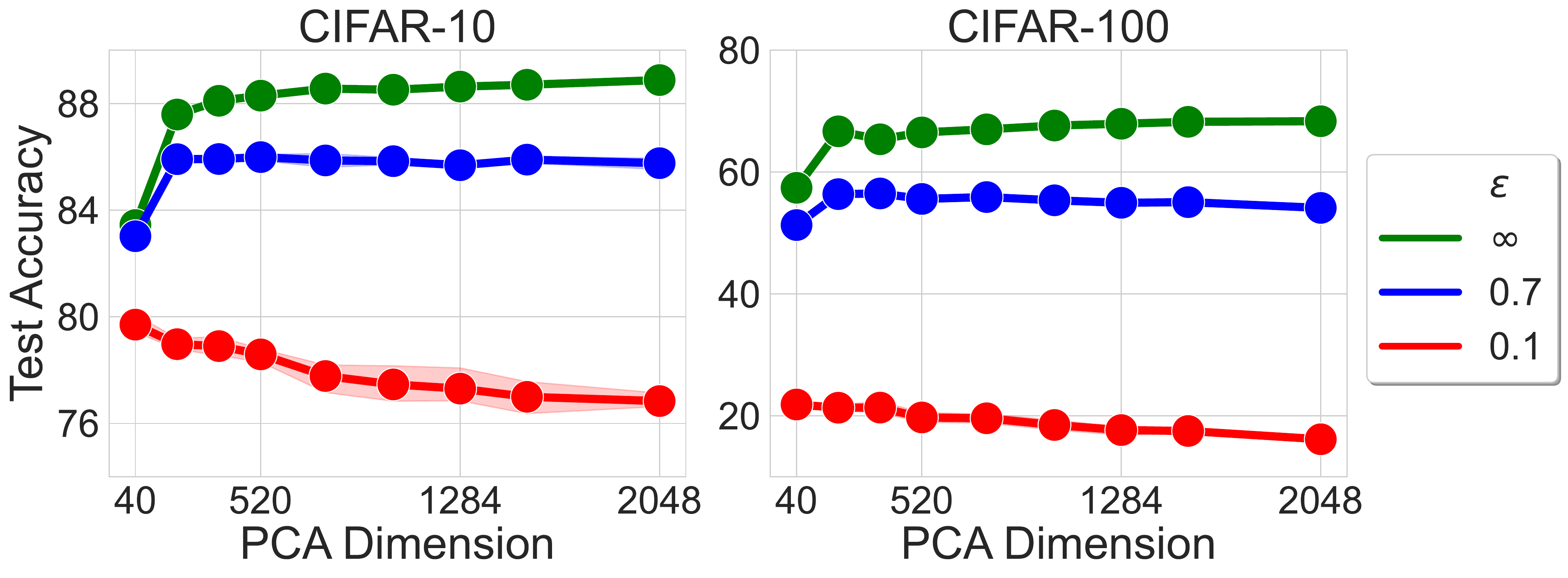}
    \subcaption{SemiSL Pretraining}
    \end{subfigure}\hfill
    \begin{subfigure}[b]{.48\linewidth}
    \includegraphics[width=\linewidth]{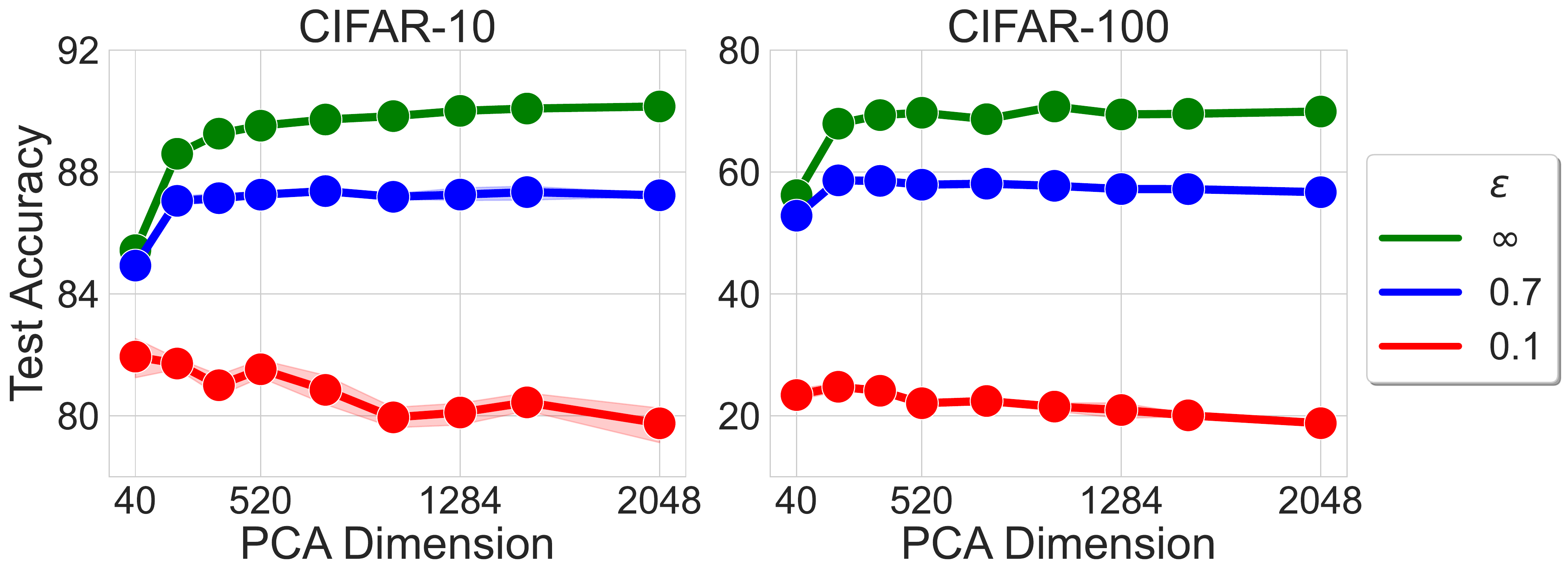}
    \caption{SemiWeakSL Pretraining}
    \end{subfigure}
    \caption{DP Training of linear classifier on different pre-trained features using the PRV accountant for CIFAR-10 and CIFAR-100.}\label{fig:neardist_app}
    \end{figure}

\begin{wrapfigure}{r}{0.35\linewidth}
    \centering
    \includegraphics[width=0.8\linewidth]{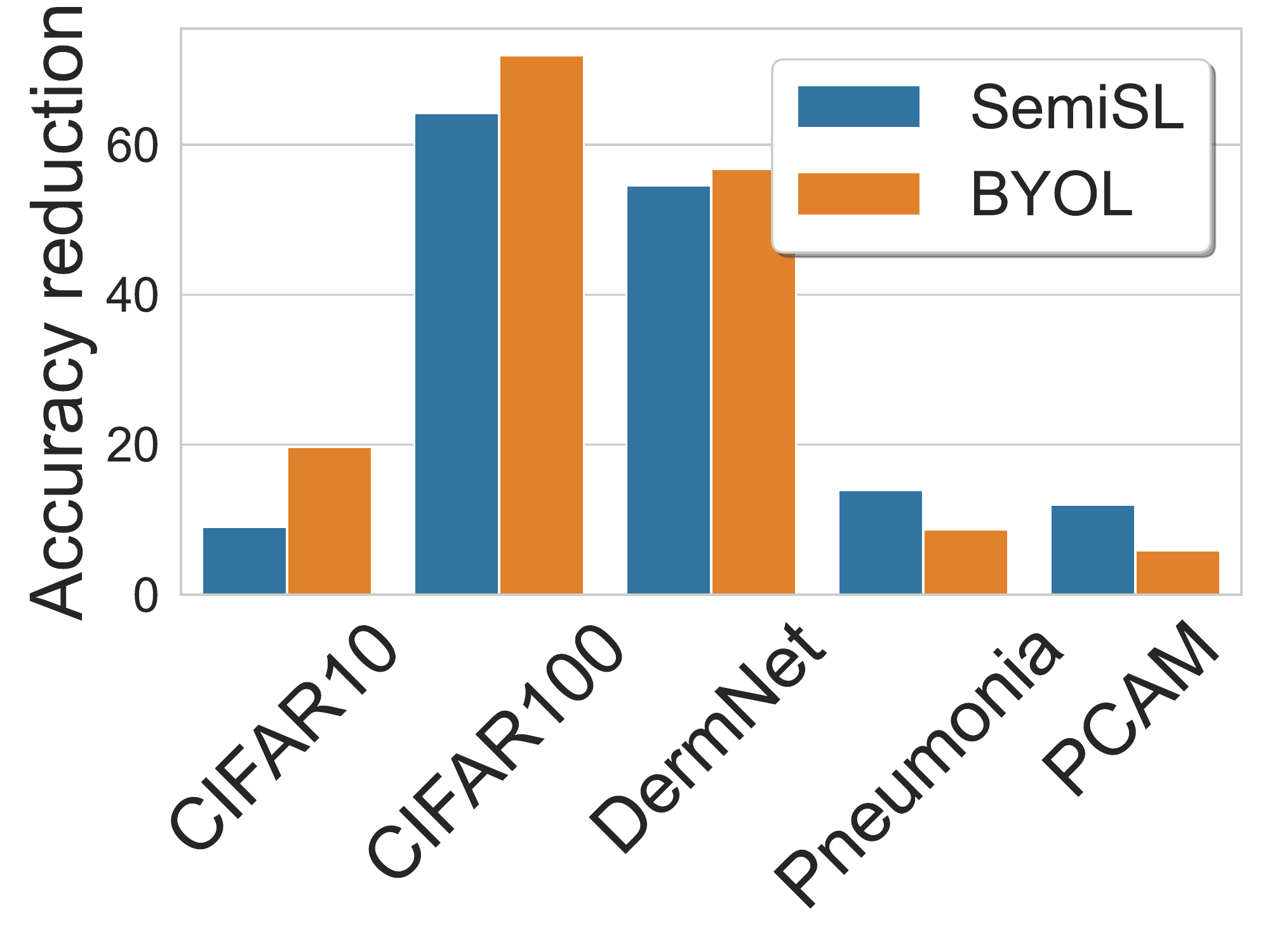}
    \caption{Comparing reduction in test accuracy for different datasets between using SemiSL and BYOL pre-trained networks.}
    \label{fig:semiSLvsSSL}\vspace{-10pt}
\end{wrapfigure}
 In~\Cref{fig:neardist,fig:farshift} in the main text, we only
reported accuracies for the best performing pre-training algorithm.
In this, section we report the performance of our algorithm against
five different pre-training algorithms that we consider in this paper.
In particular, we consider two self-supervised pre-training
algorithms: BYOL~\citep{BYOL} and MoCov2+~\citep{MocoV2+} and two
semi-supervised algorithms~\citep{SemiSL}. While one of them is a
Semi-Supervised~(SemiSL) algorithm, the other only uses weak
supervision and we refer to it Semi-Weakly Supervised~(SemiWeakSL)
algorithm. In~\Cref{fig:neardist_app} we report the results on
CIFAR-10 and CIFAR-100. In~\Cref{fig:farshift_app} we report the
results for Flower-16~\citep{Flowers16}, GTSRB~\citep{GTSRB},
PCAM~\citep{PCAM}, Pneumonia~\citep{kermany2018identifying} and
DermNet~\citep{DermNet}. 

Similar to~\Cref{fig:sslvssl} in the main text, we show the accuracy
reduction for Semi-Supervised pre-training vs BYOL~(Self-Supervised)
pre-trianing in~\Cref{fig:semiSLvsSSL}. Our results shows similar
results as~\citet{shi2022how} that labels are more useful for
pre-training for tasks where there is a significant label overlap
between the pre-training and the final task.

    \begin{figure}[t]
        \centering
        \begin{subfigure}[b]{.95\linewidth}
            \includegraphics[width=\linewidth]{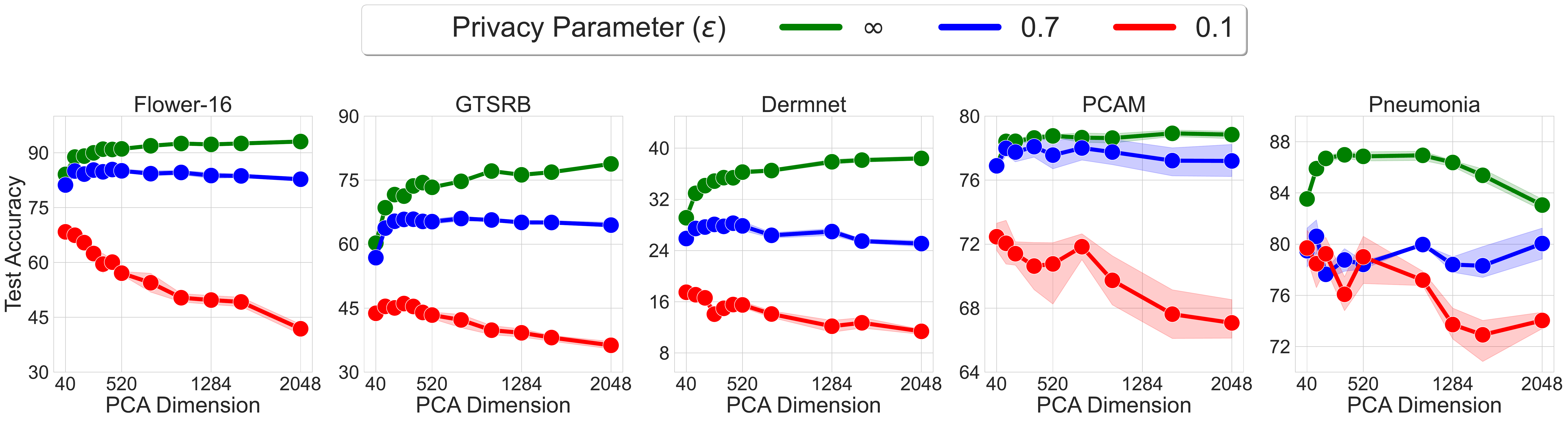}
            \subcaption{SL Pretraining}
        \end{subfigure}
        \begin{subfigure}[b]{.95\linewidth}
        \includegraphics[width=\linewidth]{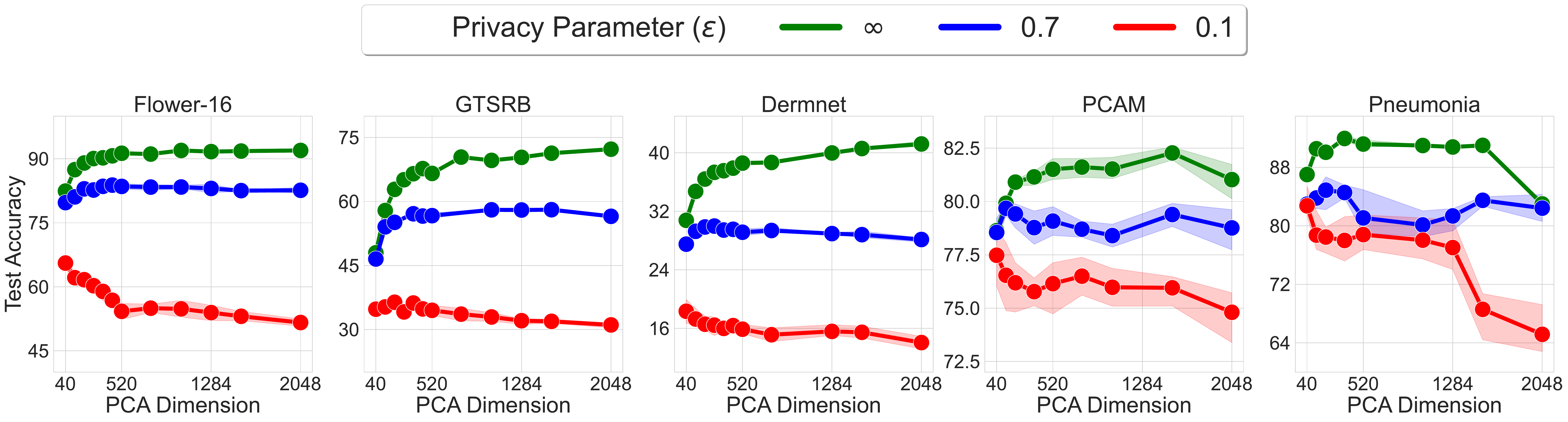}
        \subcaption{BYOL Pretraining}
        \end{subfigure}\hfill
        \begin{subfigure}[b]{.95\linewidth}
        \includegraphics[width=\linewidth]{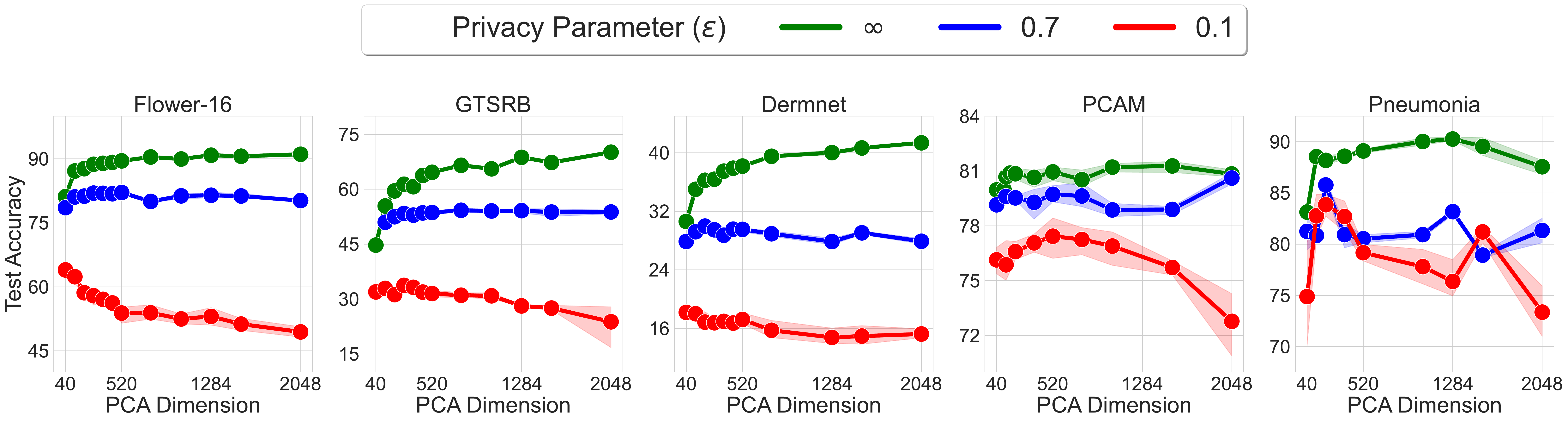}
        \subcaption{MoCov2+ Pretraining}
        \end{subfigure}
        \begin{subfigure}[b]{.95\linewidth}
        \includegraphics[width=\linewidth]{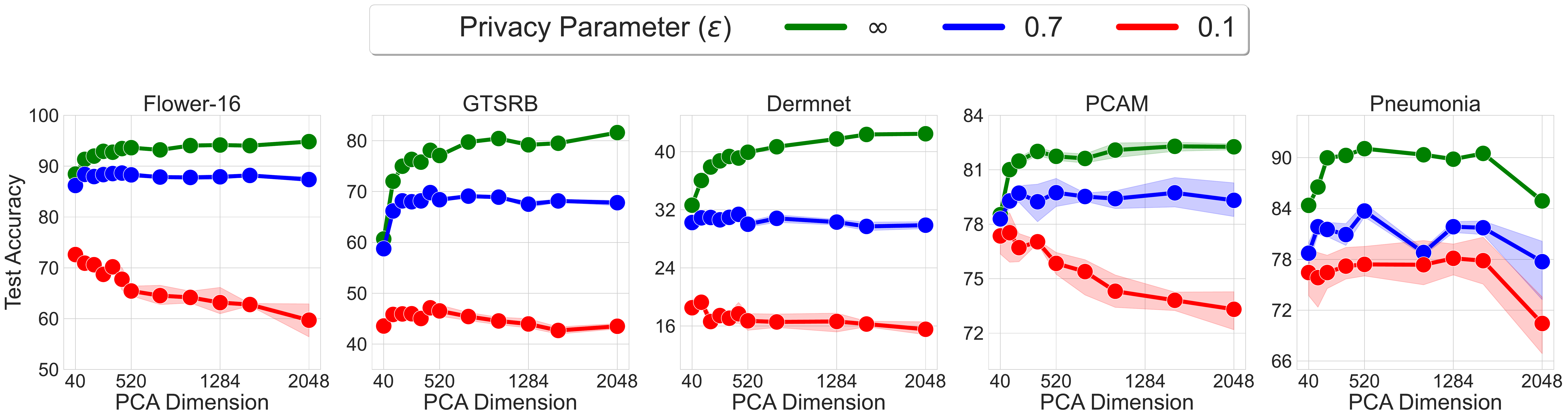}
        \subcaption{SemiSL Pretraining}
        \end{subfigure}\hfill
        \begin{subfigure}[b]{.95\linewidth}
        \includegraphics[width=\linewidth]{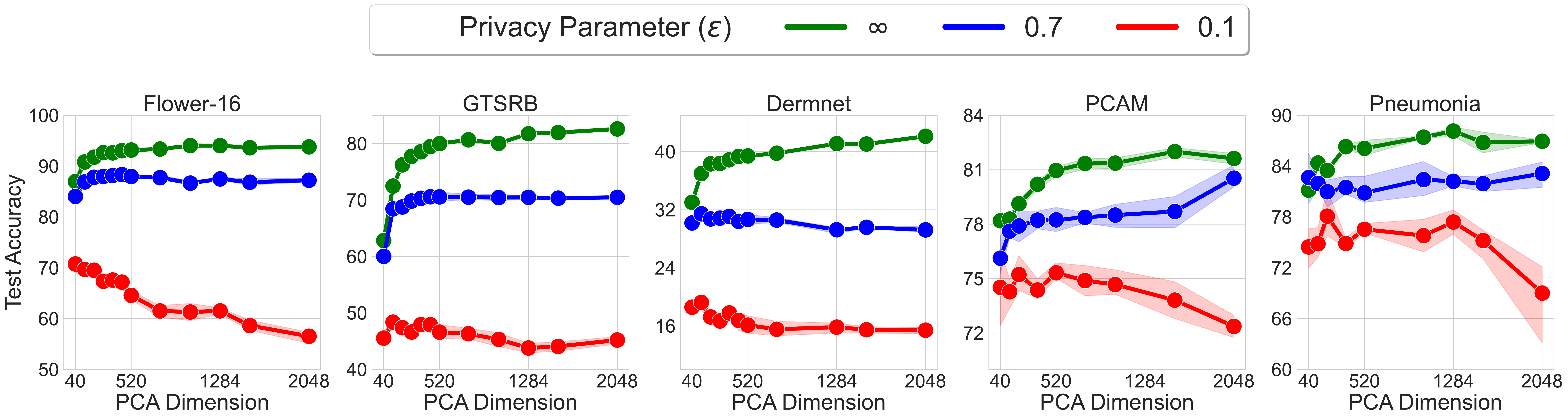}
        \caption{SemiWeakSL Pretraining}
        \end{subfigure}
        \caption{DP Training of linear classifier on different pre-trained features using the PRV accountant for Flower-16, GTSRB, DermNet, PCAM, and Pneumonia.}\label{fig:farshift_app}
        \end{figure}
\section{Computational Cost, Broader Impact and Limitations}
\label{app:additional-remarks}
\paragraph{Computational cost}
Except for the supervised training on ImageNet32x32, we leverage
pre-trained models. To optimize the training procedure, we checkpoint
feature embeddings for each dataset and pre-trained model. Therefore,
training requires loading the checkpoint and training a linear layer
via SGD (or DP-SGD), accelerating the training procedure by avoiding
the forward pass through the feature encoder.  We use a single Tesla
M40 (11GB) for each run.  We perform approximately 405 runs for
PILLAR. The same amount of runs is performed for the JL baseline.  For
GEP we perform 3528 runs. For AdaDPS we  perform 4320 runs. We make
the code needed to reproduce our main results available in the zip of
the supplementary materials. 
\paragraph{Broader impact and Limitations}
In this work we show our method can be used in order to increase the
utility of models under tight Differential Privacy constraints.
Increasing the utility for low $\epsilon$ is crucial to foster the
adoption of DP methods that provide provable guarantees for the
privacy of users. Further, unlike several recent works that have shown
improvement in accuracy for deep neural networks, our algorithm can be
run on commonly available computational resources like a Tesla M40
11GB GPU as it does not require large batch sizes. We hope this will
make DP training of high-performing deep neural networks more
accessible. Finally, we show our algorithm improves not only on
commonly used benchmarks like CIFAR10 and CIFAR100 but also in privacy
relevant tasks like medical datasets including Pneumonia, PCAM, and
DermNet. We hope this will encourage future works to also consider
benchmarking their algorithms on more privacy relevant tasks.

As discussed, the assumption that labelled public data is available
may not hold true in several applications. Our algorithm does not
require the public data to be labelled, however the distribution shift
between the public unlabelled data and the private one should not be
too large. We have shown that for relatively small distribution shift
our method remains effective. Finally, recent works have suggested
that differentially private learning may disparately impact certain
subgroups more than
others~\citep{bagdasaryan2019differential,cummings2019incompatibility,sanyalhowunfair22a}.
It remains to explore whether semi-private learning can help alleviate
these disparity.

\end{document}